\documentclass[journal, 10pt]{IEEEtran}
\usepackage{graphicx}
\usepackage{amssymb}
\usepackage{amsmath}
\usepackage{cite}
\usepackage{amsthm}
\usepackage{epstopdf}
\usepackage{color}
\usepackage{soul}
\usepackage{balance}
\usepackage[caption=false]{subfig}
\usepackage{multirow}
\usepackage{multicol}
\usepackage{booktabs}
\usepackage{amsfonts}
\usepackage{epsfig}
\usepackage{url}    
\usepackage{algorithm}
\usepackage{mathrsfs}
\usepackage{amsmath}
\usepackage{physics}
\usepackage{hyperref}
\usepackage{bbm}
\usepackage{dirtytalk}

\usepackage[dvipsnames]{xcolor}

\hypersetup{
    colorlinks=true,
    linkcolor=blue,
    citecolor=cyan,
    pdfborder={0 0 0},
}

\definecolor{duong}{rgb}{0,0,0}

\usepackage[english]{babel}
\usepackage[utf8]{inputenc}
\usepackage{textcomp}
\usepackage{algorithm}
\usepackage[noend]{algpseudocode}

\newtheorem{definition}{Definition}
\newtheorem{theorem}{Theorem}
\newtheorem{lemma}{Lemma}

\newtheorem{example}{Example}

\newtheorem{assumption}{Assumption}
\newtheorem{challenge}{Challenge}

\title{Distortion Resilience for Goal-Oriented \\ Semantic Communication}
\author{Minh-Duong~Nguyen, 
        Quang~Vinh~Do, 
        Zhaohui Yang,~\IEEEmembership{Member,~IEEE}, \\
        Quoc-Viet Pham,~\IEEEmembership{Senior Member,~IEEE,}
        and Won-Joo~Hwang,~\IEEEmembership{Senior Member,~IEEE}
\thanks{Minh-Duong Nguyen and Won-Joo~Hwang (corresponding author) are with the Department of Information Convergence Engineering, Pusan National University, Busan 46241, Republic of Korea (e-mail: \{duongnm, wjhwang\}@pusan.ac.kr).}
\thanks{Quang~Vinh~Do is with Wireless Communications Research Group, Faculty of Electrical and Electronics Engineering, Ton Duc Thang University, Ho Chi Minh City, Vietnam (e-mail: dovinhquang@tdtu.edu.vn).}
\thanks{Zhaohui Yang is with the College of Information Science and Electronic Engineering, Zhejiang University, Hangzhou 310027, China, and also with the Zhejiang Provincial Key Laboratory of Information Processing, Communication and Networking (IPCAN), Zhejiang University, Hangzhou 310027, China (e-mail:  yang\_zhaohui@zju.edu.cn).}
\thanks{Quoc-Viet Pham is with the School of Computer Science and Statistics, Trinity College Dublin, The University of Dublin, Dublin 2, D02 PN40, Ireland (e-mail: viet.pham@tcd.ie).}
\thanks{This work was supported by the National Research Foundation of Korea (NRF) grant funded by the Korea government (MSIT)(RS-2024-00336962); was supported by Institute of Information \& communications Technology Planning \& Evaluation (IITP) under the Artificial Intelligence Convergence Innovation Human Resources Development (IITP-2024-RS-2023-00254177) grant funded by the Korea government (MSIT); was supported by the MSIT (Ministry of Science and ICT), Korea, under the ITRC (Information Technology Research Center) support program (IITP-2024-RS-2023-00260098) supervised by the IITP (Institute for Information \& Communications Technology Planning \& Evaluation); was supported by the Korea Agency for Infrastructure Technology Advancement (KAIA) grant funded by the Ministry of Land Infrastructure and Transport (No. RS-2023-00256816).
Following were results of a study on the ``Leaders in INdustry-university Cooperation 3.0 Project'', supported by the Ministry of Education and National Research Foundation of Korea.
}
}

\begin{document}
\maketitle
\begin{abstract}
Recent research efforts on Semantic Communication (SemCom) have mostly considered accuracy as a main problem for optimizing goal-oriented communication systems. However, these approaches introduce a paradox: the accuracy of Artificial Intelligence (AI) tasks should naturally emerge through training rather than being dictated by network constraints. Acknowledging this dilemma, this work introduces an innovative approach that leverages the rate distortion theory to analyze distortions induced by communication and compression, thereby analyzing the learning process. Specifically, we examine the distribution shift between the original data and the distorted data, thus assessing its impact on the AI model's performance. Founding upon this analysis, we can preemptively estimate the empirical accuracy of AI tasks, making the goal-oriented SemCom problem feasible. To achieve this objective, we present the theoretical foundation of our approach, accompanied by simulations and experiments that demonstrate its effectiveness. The experimental results indicate that our proposed method enables accurate AI task performance while adhering to network constraints, establishing it as a valuable contribution to the field of signal processing. Furthermore, this work advances research in goal-oriented SemCom and highlights the significance of data-driven approaches in optimizing the performance of intelligent systems. 
\end{abstract}

\begin{IEEEkeywords}
Communication efficiency, data compression, goal-oriented semantic communication, resource allocation.
\end{IEEEkeywords}
\section{Introduction}
The emergence of sixth-generation (6G) communication technologies marks a pivotal shift from serving individual users to serving as the foundation for the expansive Internet-of-Everything (IoE) \cite{2020-IoE-Survey1}. In particular, 6G communication technologies facilitate the seamless integration of humans, machines, and objects, fostering intelligent cooperation and synergy. They actively drive the development of an inclusive and intelligent human society, as envisioned in 2019 \cite{2019-6G-Visions}.

However, traditional end-to-end information transmission systems, which rely on resource optimization at the physical layer and stable transmission protocols at the network layer, are no longer sufficient to meet the increasing demands for complex, diverse, and intelligent information transmission. This includes providing support for emerging applications such as virtual reality, holographic projection, and the Metaverse \cite{2023-SemCom-VideoConferencing, 2022-SemCom-DeepWiVe, 2020-SemCom-NetworkCompression, 2023-Sem-IGS-AGC}. Consequently, it is now imperative to establish a novel communication paradigm that efficiently caters to the evolving requirements of future wireless communication applications.

Fortunately, Semantic Communication (SemCom) \cite{2021-SemComSurvey1, 2021-SemComSurvey2, 2022-SemcomSurvey3, 2022-SemcomSurvey4} emerges as a promising architectural innovation poised to redefine communication in the IoE era. \textcolor{duong}{The ultimate objective of SemCom is to efficiently transmit the content-aware and semantic information tailored to specific tasks, thus bringing the grand vision of the IoE closer to reality. To achieve the aforementioned objective, the receiver's capability to reconstruct the original data, accepting a degree of data distortion as long as it surpasses a predefined Artificial Intelligence (AI) task performance threshold.} To this end, various studies have introduced goal-oriented SemCom approaches that leverage semantic metrics (e.g., semantic similarity, semantic completeness, and semantic accuracy) to ensure effective AI task execution within the constraints of wireless communication environments. 
\textcolor{duong}{For instance, in \cite{2021-SEM-DeepSC, 2023-SemCom-Memory, 2022-SemCom-TOMUSC}, the authors propose attention-based encoders, which employ semantic channel coding for text transmission. It maximizes data rates while preventing information loss by incorporating mutual information into the loss function, accommodating variable-length input texts and output symbols. 
Additionally, a novel performance metric, phrase similarity, resembling human judgment, was introduced to assess semantic errors. In \cite{2022-Sem-TextTransmission-RedudancyRemoval}, the author introduces an innovative approach to semantically aware speech-to-text transmission. This method is built upon a soft alignment and redundancy elimination module that utilizes attention networks. The primary objective of this approach is to remove irrelevant information, resulting in more concise representations.
In the SemCom for computer vision tasks, the metrics used vary. The studies in \cite{2022-SemCom-NTSCC, 2021-SemCom-LiteDSCC, 2022-SemCom-IBA} utilize Mean Square Error (MSE) and KL divergence between the original and reconstructed images as learning metrics to enhance the performance of semantic encoders. Additionally, the research in \cite{2024-SemCom-AdaSem} incorporates application-oriented distortion from PoseNet \cite{2021-DL-PoseNet} into the MSE, creating a joint loss function for goal-oriented SemCom aimed at camera relocalization.}

Despite numerous efforts to reduce redundancy in SemCom, there is a potential drawback in terms of information loss that significantly impairs the quality of data transmission. As a consequence, efficient resource allocation plays a pivotal role in enhancing SemCom system performance. \textcolor{duong}{As an example, the study in \cite{2022-Sem-RA-TextSem, 2023-Sem-SCRA-D2DVN, 2024-Sem-ARA-SCN, 2024-Sem-OIT-CSSN} introduces a new metric that is contingent on task accuracy, to optimize energy efficiency within SemCom networks.} In tandem, the authors of papers \cite{2022-SEM-AdaptableSemanticCompression, 2023-Sem-TOSCN, 2022-FL-CompAided, 2023-FL-SCFL} design a framework for goal-oriented multi-user SemCom. This framework empowers users to efficiently extract, compress, and transmit semantic information to the edge servers. Notably, it features an algorithm for controlling compression ratios within the SemCom network. In \cite{2022-Sem-StochasticRA}, a Stochastic Semantic Transmission Scheme (SSTS) is designed to minimize transmission and storage costs within the network, considering the uncertainty of users' demands. \textcolor{duong}{Researches presented in \cite{2022-Sem-QOERA, 2023-SemCom-EE-RSMA, 2023-Sem-RateSplitting} leverage semantic rate and semantic accuracy to define Quality-of-Experience (QoE) metrics and employ the distinctive QoE model to define the optimization problem, encompassing the number of transmitted semantic symbols, channel assignment, and power allocation.}

Although numerous researches have been made in terms of resource optimization, current goal-oriented SemCom approaches face two primary challenges:

\begin{challenge}[Inconsistent semantic metrics]
    The challenge pertains to the effectiveness of existing optimization methods for SemCom systems, which heavily rely on various semantic metrics (e.g., semantic accuracy \cite{2022-Sem-QOERA, 2023-Sem-TOSCN, 2024-Sem-OIT-CSSN}, semantic completeness \cite{2022-SEM-AdaptableSemanticCompression, 2023-Sem-ORA-SAECS}, and semantic similarity \cite{2022-Sem-TextTransmission-RedudancyRemoval, 2022-Sem-RA-TextSem, 2023-Sem-NOMA}). Unfortunately, these semantic metrics exhibit inconsistency across different tasks. Specifically, the lack of uniformity necessitates the adaptation of optimization approaches to meet the unique requirements of each task. Consequently, the existing optimization techniques for SemCom suffer from a notable limitation: \emph{they lack applicability within a complex SemCom network comprising multiple AI tasks}.
\label{challenge:1}
\end{challenge}

\begin{challenge}[Straggling dilemma]
    The challenge stems from the inherent dilemma posed by semantic metrics. Specifically, the purpose of optimization in goal-oriented SemCom system is to enhance the efficiency of distributed AI learning tasks via wireless channels \cite{2024-Sem-AdaptiveFeedBack, 2024-Sem-QKD}. However, the utilization of semantic metrics necessitates the availability of feedback obtained from completed AI training tasks \cite{2022-Sem-AttentionRL, 2023-Sem-ORA-SAECS, 2023-Sem-SCRA-D2DVN}. Consequently, current optimization techniques for SemCom face a significant drawback, i.e., their infeasibility. This is because AI training tasks for the SemCom network optimization must be completed before device performance can be labelled, resulting in a considerable straggling problem.
\label{challenge:2}
\end{challenge}
To address the two aforementioned challenges, we design a novel semantic metric called \emph{semantic distortion}. To be more specific, semantic distortion refers to the interference introduced by diverse factors through a comprehensive SemCom system, characterized by a MSE between original and reconstructed data at the transmitter and receiver, respectively. To gain a deeper insight into our work, we consider the following example of semantic distortion.
\begin{example}[Distorted Data Impact on AI Performance]
Consider an image being transmitted through a communication channel and undergoing various distortions, as depicted in Figure~\ref{fig:garnacho}. The AI model at the receiver possesses the capability to extract information from the image to identify the human identity depicted in the image. As shown in the figures, when the noise affecting the original image is minimal (Figure~\ref{fig:gauss1}), it is easy to recognize that the image portrays Garnacho, specifically a player from Manchester United. However, as the noise levels increase, the task of identifying the human subject becomes progressively challenging (as in Figures~\ref{fig:gauss2} and \ref{fig:gauss3}). 
\end{example}
\begin{figure}[!htb]
	\centering
	\subfloat[\label{fig:gauss0}]{\includegraphics[width=0.44\linewidth]{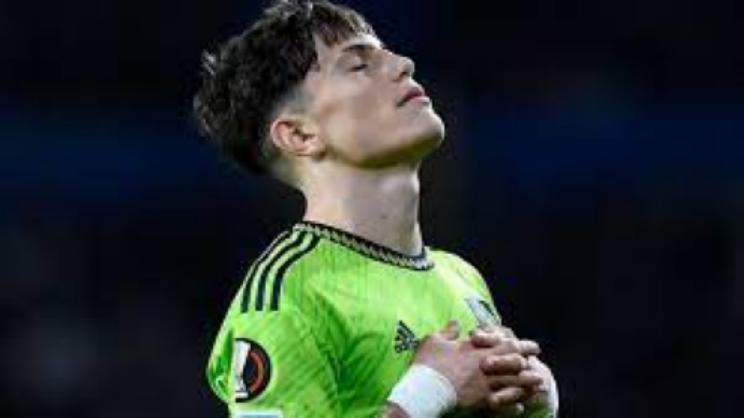}} 
	\subfloat[\label{fig:gauss1}]{\includegraphics[width=0.44\linewidth]{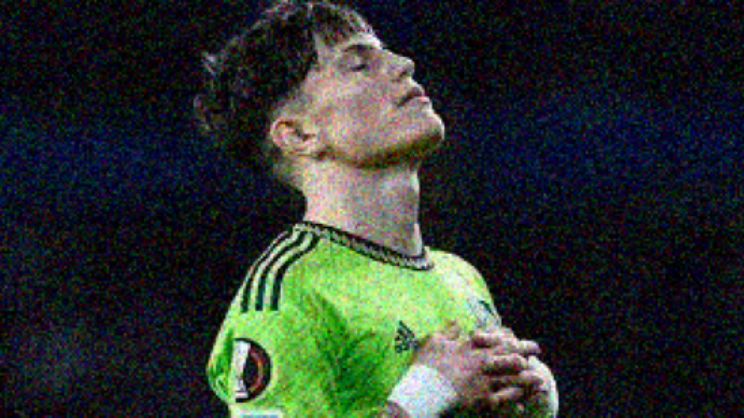}} \\
	\subfloat[\label{fig:gauss2}]{\includegraphics[width=0.44\linewidth]{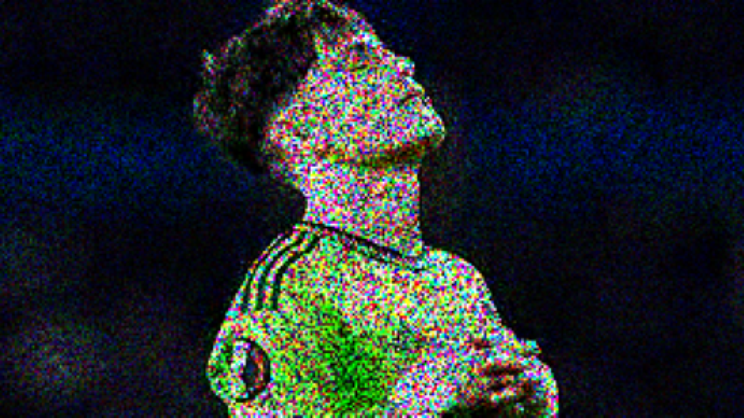}}
    \subfloat[\label{fig:gauss3}]{\includegraphics[width=0.44\linewidth]{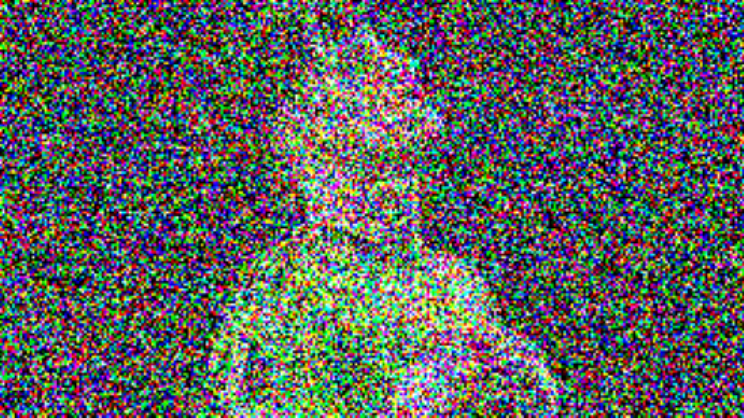}}
    \caption{Illustration of an image being affected by different distortions via the communication channel. a) the original image, b) image affected by Gaussian noise with variance of 25, c) image affected by spackle noise with variance of 100, d) image affected by Gaussian noise with variance of 250.}
\label{fig:garnacho}
\end{figure}
Similar to human cognition, AI models experience degradation when exposed to training images with noise levels exceeding a specific threshold. Therefore, \emph{by setting a specific threshold for image distortion} (effectively limiting the distortion introduced during communication), we can guarantee that our AI model at the receiver can operate and train effectively. This approach, namely Distortion Resilience for Goal-Oriented Semantic Communication (DRGO), allows us to assess the performance of a goal-oriented SemCom system while disregarding the accuracy achieved through the learning process as a factor influencing the evaluation of the communication system's efficiency, thus, enable the online goal-oriented learning capability of SemCom system. More specifically, DRGO addresses these challenges as follows.
\begin{itemize}
    \item In response to Challenge~\ref{challenge:1}, DRGO for semantic distortion provides a unifying solution that allows all other semantic metrics to be transformed into distortion. Subsequently, the semantic distortion enables the network optimization across various AI tasks within a unified SemCom system.
    \item \textcolor{duong}{In response to Challenge~\ref{challenge:2}, by adopting the distortion perspective, we can consistently evaluate the entire network system. Furthermore, DRGO approximates the degradation in goal-oriented AI performance resulting from data distortions induced by wireless channels. Consequently, we can evaluate the goal-oriented AI performance before applying the training at the receiver, thus, mitigate the straggling dilemma efficiently.}
\end{itemize}
In summary, our contributions are summarized as follows: 
\begin{itemize}
    \item We undertake a comprehensive analysis of the SemCom process by using information theory, particularly focusing on data distortions. The entire SemCom process (encompassing the transmission model, compression model, and AI task model) can be consistently and comprehensively evaluated through a unified formulation.
    \item To minimize transmission time, we develop DRGO that adaptively determines optimal compression ratios, and wireless resource allocation to meet the lower-bound requirements for task performance decisions. By framing the problem in terms of data distortions, we make it feasible to online goal-oriented learning.
    \item \textcolor{duong}{To confront the difficulties of non-convex and intricate problem, we present a DRL approach, specifically Deep Deterministic Policy Gradient with Explicit Constraints and Implicit Constraints (DDPG-EI). DDPG-EI is designed to reduce the optimization space of the DRGO problem. Consequently, DDPG-EI shows substantial enhancement compared to traditional DRL techniques.}
    \item We conduct extensive experiments to validate the efficacy of our proposed algorithm. The results reveal significant improvements in transmission rates compared to conventional communication processes while concurrently maintaining the desired performance levels of AI tasks.
\end{itemize}
This paper is structured as follows. Section~\ref{sec:system-model-fundamentals} describes the system model and preliminaries. Section~\ref{sec:problem-formulation} explains the problem formulation. Section~\ref{sec:drl-approach} presents the deep reinforcement learning approach to solve the problem. Section~\ref{sec:experimental-results} demonstrate the results, respectively. Section~\ref{sec:conclusion} concludes the paper.

\section{System Model and Fundamentals} \label{sec:system-model-fundamentals}
We consider a wireless SemCom system consisting of $U$ users and one base station (BS), as illustrated in Figure~\ref{fig:Semcom-Architecture}. Each user transmits semantic information to the BS at the network edge, which is equipped with computing capabilities to perform AI operations. \textcolor{duong}{The transmitter consists of a semantic encoder that compresses the data, and a channel encoder that generates symbols to facilitate the transmission over the wireless network. The BS acts as a receiver and is equipped with a channel decoder for symbol detection and a semantic decoder that reconstructs the data.} An AI model is integrated into the BS to perform intelligent computing according to the received data. Consequently, the BS returns the task results back to users. In our research, we categorize AI tasks into two operations: AI training process and AI execution.

\subsection{System Model}
\label{sec:system-model}
To execute the transmission of source information on each user $u$, the transmitter first extracts its meaning. This process involves sequentially encoding the original data \textcolor{duong}{$s_u\in\mathbb{R}^{d(s_u)\times 1}$, where $d(\cdot)$ represents the size of the data,} through semantic and channel encoders. Hence, \textcolor{duong}{the encoded symbols, $x_u\in\mathbb{C}^{d(x_u)\times 1}$}, can be expressed as
\begin{align}
    x_u = f_{\alpha}\left(f_{\beta}(s_u)\right),
\label{eq:encoder}
\end{align}
\textcolor{duong}{where $f_{\beta}:\mathbb{R}^{d(s_u)\times 1}\rightarrow\mathbb{R}^{d(z_u)\times 1}$ represents the deep semantic encoder with the parameter set $\mathbf{\beta}$ that encodes the original data into data representations $z_u$, and $f_{\alpha}:\mathbb{R}^{d(z_u)\times 1}\rightarrow\mathbb{R}^{d(x_u)\times 1}$ represents the channel encoder characterized by the parameter set $\mathbf{\alpha}$.} If $x_u$ is transmitted, the signal is affected by the wireless channel, denoted as $P(\hat{x}_u\vert x_u)$ (i.e., the distributional probability of \textcolor{duong}{received signal $\hat{x}\in \mathbb{C}^{d(x_u)\times 1}$} given the transmitted signal $x_u$). The transmission process from the users to the BS can be modeled as
\begin{align}
    \hat{x}_u = \mathbf{h}_u\cdot x_u + \mathbf{n}_u
\end{align}
where $\mathbf{h}_u$ represents the Rayleigh fading channel, and $\mathbf{n}_u \sim \mathcal{C}\mathcal{N}(0, \sigma^2_{\textrm{model},u})$ denotes Independent and Identically Distributed (IID) Gaussian noise vector with variance $\sigma^2_{\textrm{model},u}$ for each channel between user $u$ and the BS. \textcolor{duong}{At the BS, the received data is firstly decoded by the channel decoder $g_{\alpha}:\mathbb{R}^{d(x_u)\times 1}\rightarrow\mathbb{R}^{d(z_u)\times 1}$ and then by the semantic decoder $g_{\beta}:\mathbb{R}^{d(z_u)\times 1}\rightarrow\mathbb{R}^{d(s_u)\times 1}$.} The compressed data is decoded via the semantic decoder as follows: 
\begin{align}
    \mathbf{\hat{s}} = g_{\beta}(g_{\alpha}(\hat{x}_u)).
\label{eq:decoder}
\end{align}
After the data is decoded at the BS, it is employed for either AI training (during the AI training process) or AI inference (during the operational process). 
\subsection{Communication Model}
\begin{figure}[t]
\centerline{\includegraphics[width=1\linewidth]{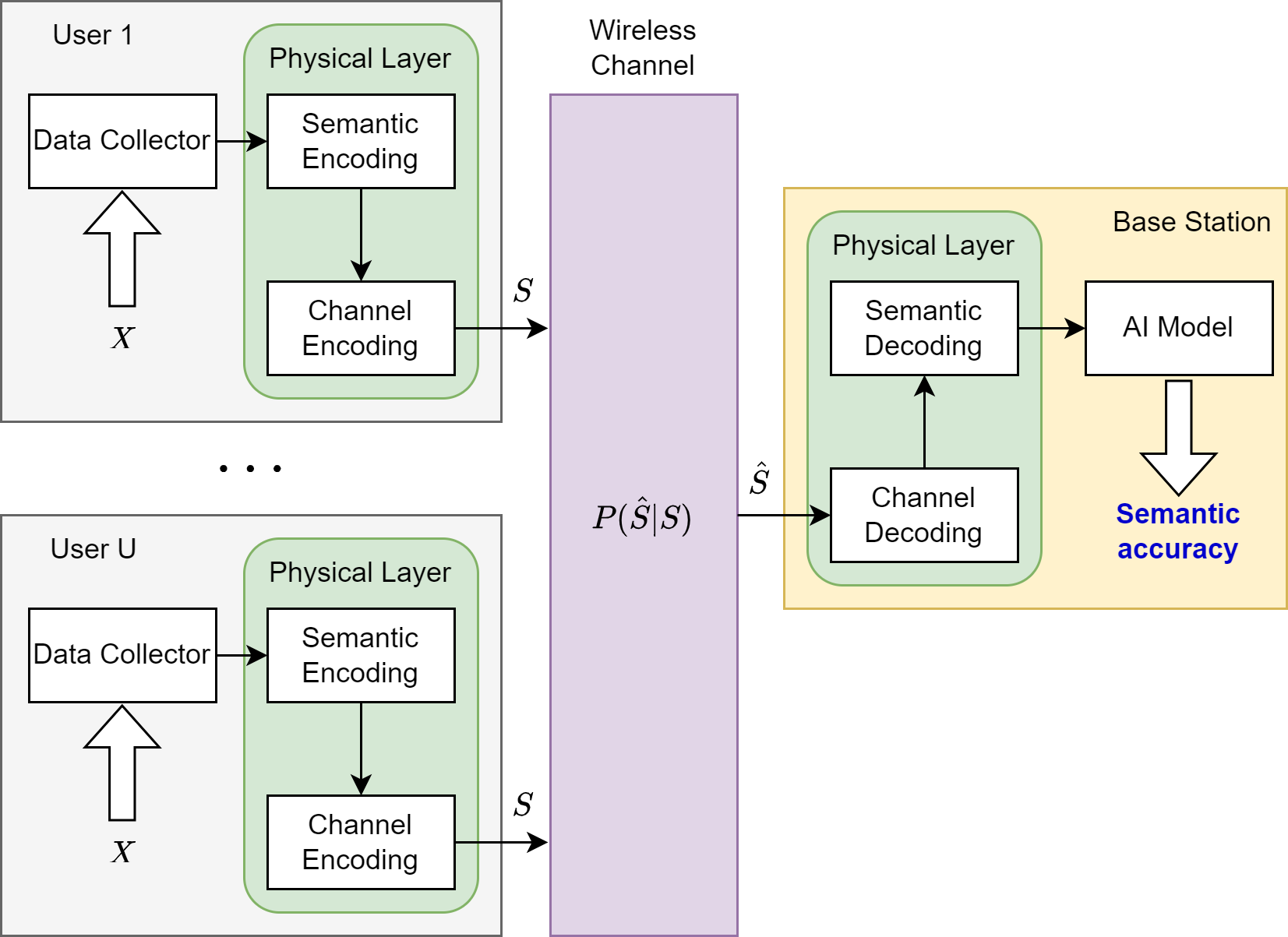}}
\centering
\caption{System model of the considered SemCom.}
\label{fig:Semcom-Architecture}
\end{figure}
The maximum achievable data rate between user $u$ and the BS can be denoted as:
\begin{align}
R (B_u, P_u)
&= B_u \log_{2}\left(1+\frac{h_u P_u}{I_n+ B_u N_0}\right),
\end{align}
where $B_u, P_u$ denote the bandwidth and transmit power of user $u$, respectively. $N_0$ is the noise power spectral density. $h_u=r_u d_u^{-2}$ represents the channel gain, where $d_u, r_u$ are the distance and Rayleigh fading parameter between user $u$ and the BS, respectively.

\subsubsection{Compression Model} As discussed in Section~\ref{sec:system-model}, we consider a semantic encoder/decoder model that compresses the data into compact and structured representations. To assess the performance of the semantic encoder/decoder, we use two main features for each user $u$: the compression ratio and the reconstruction rate, denoted as $o_u$ and $\sigma_{\textrm{sem},u}$, respectively. Regarding the compression ratio, we employ the deep neural network (DNN) for data compression \cite{2023-FL-HCFL}, where the compression ratio $o_u$ can be characterized as the ratio between the dimensionality of the data representation and that of the original data:
\textcolor{duong}{
\begin{align}
    o_u = d(z_u)/d(x_{u}).
\label{eq:compression-ratio}
\end{align}}
To approximate the reconstruction rate, we leverage the MSE between reconstructed data $\hat{x}_u$ and original data $x_u$, i.e., $\textrm{MSE}(x_u, \hat{x}_u) = \frac{1}{N}\sum^N_{i=1}(x_{i,u} - \hat{x}_{i,u})^2$. As discussed in \cite{2023-FL-HCFL}, the DNN can be considered as restricted Boltzmann machine (RBM), and thus, the noise can be modeled by a Gaussian distribution $\mathcal{N}(0,\sigma^2_{\textrm{sem}})$ with mean $0$ and variance $\sigma^2_{\textrm{sem}}$. Consequently, we have the following relationship
\begin{align}
    \sigma_{\textrm{sem},u} = \sigma_{\textrm{sem}} (o_u) &= \sqrt{\frac{1}{N}\sum^N_{i=1}(x_{i,u} - \hat{x}_{i,u})^2} \notag \\
    &= \sqrt{\textrm{MSE}(x_u, \hat{x}_u)},
\end{align}
where the variance $\sigma^2_{\textrm{sem},u}$ can be considered as a non-linear function affected by the compression ratio $o_u$ (i.e., $\sigma^2_{\textrm{sem},u}=f(o_u)$). \textcolor{duong}{Traditionally, the relationship between $\sigma^2_{\textrm{sem},u}$ and $o_u$ is mutually dependent, implying that the distortion rate varies depending on the specific compression technique employed.} Consequently, establishing a universal approximation for $f(o_u)$ is challenging. Nevertheless, it should be noted that the compression technique is chosen before setting up the SemCom system. As a result, we can select the compression ratio from a predefined table (an example of such a table can be found in Section~\ref{sec:semantic-setting}).

\subsubsection{Information Transfer} \textcolor{duong}{In each communication round, users transmit $Z_u$ bits data to the BS for the AI process. We denote $o_u$ as a compression ratio on user $u$. Subsequently, user $u$ transmits an amount of $o_u Z_u$ bits in each communication round. As a result, the transmission time for user $u$ can be calculated as
\begin{align}\label{LKD-surrogateOutput}
T_\textrm{sem}^u (B_u, P_u, o_u, \beta_u) = 
\begin{cases}
    \frac{o_u Z_u}{R(B_u, P_u)}, & \text{if } \beta_u\neq 0, \\
    0, & \text{otherwise}.
\end{cases}
\end{align}
where $\beta_u=0$ indicates that user $u$ is not assigned to send the data to the server, and otherwise. We define the average time transmission as follows: 
\begin{align}
    T_\textrm{avg} = \frac{1}{U} \sum^{U}_{u=1}T_\textrm{sem}^u (B_u, P_u, o_u, \beta_u).
\end{align}
The total transmission energy consumption at each user $u$ for sending data is given by
\begin{align}
    E_\textrm{sem}^u (B_u, P_u, o_u, \beta_u) = P_u T_\textrm{sem}^u (B_u, P_u, o_u, \beta_u).
\end{align}
Hence, the total communication energy of the system can be expressed as
\begin{align}
    E_\textrm{tot} 
    = \sum^U_{u=1} E_\textrm{sem}^u (B_u, P_u, o_u, \beta_u).
\end{align}}
\subsection{Goal-Oriented Deep Learning Model}
Apart from other goal-oriented SemCom research, we consider a general AI system that is capable of multi-task learning (e.g., Natural Language Processing (NLP), Image Processing (IP), Time Series Prediction (TSP)). The objective of the AI task-based SemCom system is to retrieve the output $\mathbf{t}$ with respect to the task's requirement. For instance, the output could encompass an image in the case of image retrieval, a label in the context of image classification, or text for speech recognition, depending on the nature of the task.
\begin{definition}[AI Inference]
    Acording to \cite{2007-DL-GreedyLayerWise}, the output probability $p(\mathbf{t})$ can be expressed as
\begin{align}
    p(\mathbf{t}) = p(\mathbf{t}\vert \mathbf{\hat{s}}) p(\mathbf{\hat{s}}),
\end{align}
where $p(\mathbf{t}\vert \mathbf{\hat{s}})$ is the posterior probability. 
\end{definition}

\begin{definition}[AI Training] 
    Consider the AI training task at the BS, the AI model employs a gradient descent algorithm according to a specific loss function, denoted as $F(w; \cdot)$, which can be utilized with any metric, e.g., MSE: $F(w; x_u) = \Vert \mathbf{t}-\mathbf{\hat{s}} \Vert^2$, Cross Entropy: $F(w; \mathbf{\hat{s}}) = -f\sum_{\mathbf{\hat{s}}\in S} p(\mathbf{t}) \log (\mathbf{\hat{s}})$.     
\end{definition}

\subsubsection{Semantic Distortion} \textcolor{duong}{To analyze the entire online goal-oriented SemCom system \emph{without the straggling effect} (as mentioned in Challenge~\ref{challenge:2}), where the model at the receiver has to queue the received data to execute the AI procedure and obtain the resulting AI metrics), we treat it as a sequential stochastic process.} In this regard, we introduce the following lemma concerning the sequential distortion: 
\begin{lemma}
The total distortion of the consequential Gaussian model for each user $u$ is given by
\begin{align}
\sigma_{\textrm{tot},u}^2 = \sum\nolimits^{L}_{l=1} \sigma^2_{l,u},
\end{align}
where $l$ denotes the index of the distortion generated components in the SemCom system.
\label{lemma:sequential-model-distortion}
\end{lemma}
\textit{Proof:} The proof is demonstrated in Appendix~\ref{appendix:sequential-model-distortion}.

Lemma~\ref{lemma:sequential-model-distortion} establishes a connection between a system comprising an arbitrary number of Gaussian models and the total distortion it incurs. As a result, we can formulate a joint distortion function that accommodates various network configurations. In our research, we introduce a straightforward system model for the SemCom system, which is elaborated upon in Section~\ref{sec:system-model}. Thus, we have the following Theorem. 
\begin{theorem}
\label{theorem:sequential-semcom-distortion}
   The total distortion of the SemCom Process on user $u$ is approximated by: 
    \begin{align}
        \sigma^2_{\text{tot},u} =\sigma^2_{\text{sem},u} + \sigma^2_{\text{model},u} +\sigma^2_{\text{data},u}
    \end{align}
\end{theorem}
\begin{proof}
    The proof is demonstrated in Appendix~\ref{appendix:sequential-semcom-distortion}.
\end{proof}

Theorem~\ref{theorem:sequential-semcom-distortion} shows us that the entire SemCom system is affected by the three components: 1) semantic distortion (i.e., the noise $\sigma_{\textrm{sem},u}$ induced by the information loss in lossy compression on each user $u$), 2) channel distortion (i.e., the noise $\sigma_{\textrm{model},u}$ induced by the channel model), and 3) data distortion (i.e., the noise $\sigma_{\textrm{data},u}$ induced by the data sampling process).

\subsubsection{Deep Learning Model Degradation}
To analyze the impact of distortion on AI model degradation, we make the following two assumptions regarding the loss function. 
\begin{assumption}[$L$-smooth]
    Function $F$ is $L$-Lipschitz smooth, i.e., $\nabla^2 F \leq L\mathbf{I}$.
\label{assumption:L-Lipschitz}
\end{assumption}
\begin{assumption}[$\mu$-strongly convex]
    Function $F$ is $\mu$-strongly convex, i.e., $\nabla^2 F \geq \mu\mathbf{I}$.
\label{assumption:mu-convex}
\end{assumption}
Here, we denote $\mathbf{I}$ as an identical matrix. We also have the following assumption about the ideal data.
\begin{assumption}
Given the global dataset $\mathcal{D} = \{ (x_i,y_i) \vert y_i \in \{1,2, \dots, C\} \}$, there always exists a canonical data point $\widetilde{x} = \{\widetilde{x}^c \vert c \in \{1,2,\dots, C\}\}$, with respect to data region $c$ which satisfies $\widetilde{x}^c= \lim_{j \rightarrow \infty} \sum^{j}_{i=1} (x_i \vert y_i = c) $.
\label{assumption:dataset-center}
\end{assumption}
These assumptions are widely used to prove the convergence in distributed learning \cite{2023-FL-FedEXP, 2020-FL-FedNova} in particular and in conventional machine learning in general \cite{2020-SGD-TigherTheory-IID}, thus, suitable for evaluating the goal-oriented SemCom system.

\subsubsection{Training phase} The goal-oriented AI model utilizes data collected from all users $u\in U$ for the training process. Therefore, the training data batch is sampled from a data pool, which represents the joint distribution of all distributed domains. To characterize the total distortion of the system with $U$ users, we consider the following lemma. 
\begin{lemma} \label{lemma:joint-distortion-rate}
    Considering that each user $u\in U$ samples data from different domains and the learning data at the server is sampled from a data pool, thus we can consider the joint data distortion at the server as follows:
    \begin{align}
    \sigma^2_\textrm{tot} 
    &= \frac{1}{D} \sum^U_{u=1} D_u \sigma^2_{\textrm{tot},u}, 
    \end{align}
    where $D_u$ and $D=\sum^U_{u=1}D_u$ are the number of data stored on user $u$ and of all users, respectively. $\sigma^2_{\textrm{tot},u}$ is the total distortion on each channel (from user $u$ to the BS).
\end{lemma}

Applying Assumptions~\ref{assumption:L-Lipschitz}, \ref{assumption:mu-convex}, and \ref{assumption:dataset-center}, and Lemma~\ref{lemma:joint-distortion-rate}, we can establish the theorem regarding AI training divergence when the data is distorted due to the wireless SemCom system.
\begin{theorem}[Model training degradation evaluation]\label{theorem:predict-model-variance}
    Considering $\sigma_\textrm{tot}^2$ as the average data distortion to the global ideal data $\widetilde{x}$ in the dataset $\mathcal{D}$, the learning error bound of the distorted data on the hypothesis $h_\theta$ after $N$ training rounds is demonstrated as follows:
\begin{align}
    F(w^n;\hat{x})-F(w^n;x)\leq N\left( \eta^2\frac{L}{2}-\eta \right)(L\sigma_\textrm{tot})^2,
\end{align}
where $\sigma_\textrm{tot}$ represents the expected total distortion of the SemCom system, encompassing all $U$ users. 
\end{theorem}
\textit{Proof:} The proof is demonstrated in Appendix~\ref{appendix:predict-model-variance}.

\textcolor{duong}{Theorem~\ref{theorem:predict-model-variance} assists us in estimating the degradation in AI training performance caused by noisy data induced by semantic noise. By utilizing data distortion as a proxy for approximating AI performance, we can anticipate AI performance before the completion of data transmission (because the optimization problems do not directly rely on AI metrics, i.e., accuracy or semantic similarity). Consequently, it lays the foundation for optimizing DNN in the context of goal-oriented SemCom.}

\subsubsection{Inference phase} In the inference phase, we can only have access to a limited amount of data (e.g., a single data point). Consequently, we can not use statistical measurements to evaluate the generalization gap, as is done during the training phase. To assess the inference capacity of the AI model under a noisy channel, we adopt the following lemma.
\begin{lemma}[Total variation of distorted data]
    Given the data $x$ and the noisy data $\hat{x}$ induced by the noisy channel, the total variation between the distributions of two data can be computed as follows: 
    \begin{align}
        \textrm{TV}(x, \hat{x}) \leq \frac{1}{\sqrt{2\pi}\sigma} \exp{-\left(\frac{W}{2\sigma}\right)^2},
    \end{align}
    where $W$ is the decision boundary of data changes that can preserve the goal-oriented AI task, $\sigma$ denotes the distortion between the original data and the noisy data. Thus, we can measure the distribution shift between the two data based on the degree of distortion of data components that cross the decision boundary.
\label{lemma:total-variation}
\end{lemma}
\textcolor{duong}{For a more intuitive understanding, consider a basic task such as identifying if there is a person in a Figure~\ref{fig:gauss3}, AI model can adapt to notably distorted images without difficulty. However, for a challenging task (e.g., determining the identity of the person), AI model needs data with clearer details. In case of image with significant distortions (as depicted in Figure~\ref{fig:gauss3}), AI model may struggle to recognize the person accurately. This demonstrates that the complexity of the task influences the extent to which the AI can manage image distortions. Therefore, we have the following theorem.}
\begin{theorem}[Model inference degradation evaluation]
\label{theorem:inference-divergence}
    The inference gap of the goal-oriented AI model under the noisy channel can be represented as: 
    \begin{align}
        \mathbb{E}\left[p(\mathbf{\hat{t}}) - p(\mathbf{t})\right]
        \leq \frac{1}{\sqrt{2\pi}\sigma} p(\mathbf{t}\vert s_u) \exp{-\left(\frac{W}{2\sigma}\right)^2}
    \end{align}
\end{theorem}
\textit{Proof:} The proof is demonstrated in Appendix~\ref{appendix:inference-divergence}.

Through the utilization of Theorem~\ref{theorem:inference-divergence}, we gain the capability to predict the degradation of AI model performance resulting from input data distortion in comparison to the original data. The salient aspect of Theorem~\ref{theorem:inference-divergence} lies in its ability to gauge the deterioration of the AI model on a per-data-point basis. Consequently, during the online network optimization (while the AI task is in progress), there is no longer a need to collect a substantial quantity of data samples to approximate the shift in data distribution (as required in the case of Theorem~\ref{theorem:predict-model-variance}). Moreover, by anticipating how data distortion affects AI performance, we can circumvent the need for employing existing semantic metrics (e.g., semantic similarity, semantic relatedness) which demand AI inference for evaluation. As a result, we can alleviate the straggling challenges that afflict modern research in goal-oriented SemCom.
\begin{table}[]
\centering
\caption{Table of abbreviations and notations.}
\begin{tabular}{|c|l|l}
\cline{1-2}
\textbf{Notation} & \multicolumn{1}{c|}{\textbf{Description}} & \\ \cline{1-2}
$U$      & Number of users                                    & \\ \cline{1-2}
$B_u$    & Bandwidth of user $u$                              & \\ \cline{1-2}
$o_u$    & Compression ratio of user $u$             & \\ \cline{1-2}
$D$      & Number of data stored on server's data pool        & \\ \cline{1-2}
$D_u$    & Number of data of user $u$ stored on server        & \\ \cline{1-2}
$B_\textrm{max}$    & Maximum total bandwidth                 & \\ \cline{1-2}
$P_\textrm{max}$    & Maximum total transmit power            & \\ \cline{1-2}
$F(w;\cdot)$        & Goal-oriented AI loss function          & \\ \cline{1-2}
$d(\cdot)$        & Data size                        & \\ \cline{1-2}
$\beta_u$    & User selection vector                 & \\ \cline{1-2}
$s_u$        & Source data on user $u$               & \\ \cline{1-2}
$\hat{s}_u$  & Reconstructed data on user $u$        & \\ \cline{1-2}
$x_u$       & Encoded data on user $u$               & \\ \cline{1-2}
$\hat{x}_u$  & Received encoded data on user $u$     & \\ \cline{1-2}
\end{tabular}
\label{tab:notations}
\end{table}
\section{Problem Formulation}
\label{sec:problem-formulation}
Our goal is to create a data distortion-driven goal-oriented SemCom system that achieves two primary objectives: 1) preserving consistent semantic metrics (Challenge~\ref{challenge:1}) and 2) mitigating the straggling issue (Challenge~\ref{challenge:2}) in optimization of SemCom system. To achieve this, we establish two optimization problems that focus on optimizing the SemCom system during the AI training and AI inference.

\subsection{Optimization Problem for Semantic AI Training/Inference}
\label{sec:sem-AI-trainer}
\textcolor{duong}{In data distortion-driven goal-oriented SemCom system, we consider the learning and predicting performance of the AI model. The principle of our method is the minimization of total communication energy while simultaneously preserving the AI learning or inference performance. To this end, we set an upper bound on the received data distortion at the BS, ensuring that the AI tasks do not deviate beyond a certain predefined threshold. Specifically, the optimization problem is given as.}
\begin{subequations}
	\label{subeqn-opt-pro-training:opt-pro-main}
	\begin{alignat} {3}
		& \min_{\mathbf{B}, \mathbf{P}, \mathbf{o}, \mathbf{\beta}} 
		&	 &  ~E_\textrm{tot}(\mathbf{B}, \mathbf{P}, \mathbf{o}, \mathbf{\beta})=\sum^{U}_u E_\textrm{sem}^u (B_u, P_u, o_u, \beta_u),
		\label{opt-pro-training:opt-pro} \\
		& ~~\text{s.t.}
		&   &   B_u \geq 0, \forall u \in U,
		\label{subeqn-opt-pro-training:min-bandwidth-constraint}\\
		&   &   & \sum^U_{u=1} \beta_u B_u \leq B_\textrm{max}, \forall u \in U,
		\label{subeqn-opt-pro-training:max-bandwidth-constraint}\\
		&   &   & 0 \leq P_u \leq P_\textrm{max}, \forall u \in U,
		\label{subeqn-opt-pro-training:min-power-constraint}\\		
		&   &   & 0 < o_u < 1,
		\label{subeqn-opt-pro-training:compression-ratio}\\	
		&   &   & \sigma_{\textrm{sem},u} = f(o_u),
		\label{subeqn-opt-pro-training:compression-distort}\\	
		&   &   & \beta_u \in \{0,1\},
		\label{subeqn-opt-pro-training:resource-allocation}\\
		&   &   &  f_\textrm{AI}(\sigma_{\textrm{tot}}) < \varepsilon_\textrm{target}.
		\label{subeqn-opt-pro-training:goal-oriented}
    \end{alignat}
\end{subequations}
where $\sigma_{\textrm{sem},u}$ is the distortion induced by the communication link from user $u$ to the BS. Here, we denote the optimization vectors of power allocation, bandwidth allocation, compression ratio, user selection as $\mathbf{\beta} = \{ \beta_1, \beta_2, \ldots \beta_U\}$, $\mathbf{P} = \{P_1, P_2, \ldots P_U\}$, $\mathbf{o} = \{ o_1, o_2, \ldots, o_U\}$, $\mathbf{B} = \{B_1, B_2, \ldots, B_U\}$, respectively. Constraint~\eqref{subeqn-opt-pro-training:max-bandwidth-constraint} indicates that the bandwidth allocated to chosen users must not exceed the system bandwidth. Constraint~\eqref{subeqn-opt-pro-training:min-bandwidth-constraint} are the minimum bandwidth. Constraints \eqref{subeqn-opt-pro-training:min-power-constraint} reveals the minimum and maximum power constraints, respectively. Constraint~\eqref{subeqn-opt-pro-training:compression-ratio} is the compression ratio constraint which is bounded by $1$ (indicating data without compression). Constraint~\eqref{subeqn-opt-pro-training:compression-distort} expresses relationship between the semantic distortion and the compression ratio of a specific compression algorithm. This relationship can be approximated using a function or a mapping table. Notably, we deploy and evaluate the compression algorithm before deploying the DRGO system. Constraint~\eqref{subeqn-opt-pro-training:goal-oriented} is the distortion resilience constraint. Specifically, it ensures that the total data distortion remains below a specified threshold corresponding to achieving a minimal accuracy of the task predictor, $\varepsilon_\textrm{target}$. The $f_\textrm{AI}(\sigma_{\textrm{tot}})$ is the accuracy approximation function based on Theorems~\ref{theorem:predict-model-variance} and \ref{theorem:inference-divergence}, and is defined as follow: 
\begin{align}\label{eq:AI-approx}
f_\textrm{AI}(\sigma_{\textrm{tot}}, k) = 
\begin{cases}
    \left( \eta^2\frac{L}{2}-\eta \right)(L\sigma_\textrm{tot})^2, & \text{if $k=0$} \\
    \frac{1}{\sqrt{2\pi}\sigma} p(\mathbf{t}\vert s_u) \exp{-\left(\frac{W}{2\sigma_\textrm{tot}}\right)^2}, & \text{if $k=1$},
\end{cases}
\end{align}
where the condition $k=0$ means that the optimization problem is applied to the AI training process, whereas $k=1$ indicates that the optimization is applied to the AI inference process. The AI accuracy function defined in \eqref{eq:AI-approx} is approximated through averaging over $N$ communication rounds. Consequently, apart from the Theorem~\ref{theorem:predict-model-variance}, we calculate the average over each communication round of the AI training approximation function, yielding: 
\begin{align}
    f_\textrm{AI}(\sigma_{\textrm{tot}}, k) &= \frac{1}{N} \Big[N\left( \eta^2\frac{L}{2}-\eta \right)(L\sigma_\textrm{tot})^2\Big] \notag\\
    &= \left( \eta^2\frac{L}{2}-\eta \right)(L\sigma_\textrm{tot})^2.
\end{align}
The main notations of this paper are summarized in Table~\ref{tab:notations}.

\section{DDPG-EI: Deep Reinforcement Learning Approach for DRGO}
\label{sec:drl-approach}
In this section, we briefly introduce the structures and process of our proposed DRL algorithm. To employ DRL to solve the minimization problem, it is necessary to redesign the optimization problem, ensuring that the algorithm process is properly aligned with the DRL operating rule. \textcolor{duong}{In Section~\ref{sec:DDPG}, we introduce the architecture of Deep Deterministic Policy Gradient (DDPG) \cite{2014-RL-DDPG}. In Section~\ref{sec:closed-form}, we first introduce a close-from expression for optimization constraints. From then on, we can re-design the DDPG into new algorithm, namely DDPG-EI, which shows the significant robustness over DDPG.}
\subsection{DDPG Architecture}\label{sec:DDPG}
\subsubsection{An Overview of DDPG}
DDPG is a widely used technique in DRL. It employs four DNNs, which are described as below.
\begin{itemize}
\item The actor network, also known as the policy network, is parameterized by $\phi_\mu$. It takes the state $s$ as input and produces an action denoted by $\pi(s\vert \phi_\mu)$.
\item The target actor network, parameterized by $\phi_{\mu_t}$, generates the target policy $\pi(s\vert \phi_{\mu_t})$.
\item The critic network, also referred to as the $Q$ network, is parameterized by $\phi_q$. It takes both the state $s$ and the action $a$ as inputs and outputs the corresponding action-value function $Q(s,a\vert \phi_q)$.
\item The target critic network, parameterized by $\phi_{q_t}$, produces the target action-value function $Q(s,a\vert \phi_{q_t})$.
\end{itemize}
\subsubsection{Training of the Critic Network}  
The batch of observations is sampled from the experience replay buffer $D$. The target value $Q$ of the $k^\textrm{th}$ observation $(s_k, a_k, r_k, s_{k+1})$ is represented by $y_k$ as follows:
\begin{align}
y_k = r_k + Q(s_{k+1}, a_{k+1}\vert \phi_q),
\end{align}
where $a_{k+1}$ is the action sampled by the DRGO network at step $k+1$ via the policy $a_{k+1}\sim\pi(s_{k+1}\vert \phi_\mu)$ under the influence of $s_{k+1}$ through the target actor network. The subscript $Q(s_{k+1}, a_{k+1}\vert \phi_q)$ defines the output of the target critic network. We obtain the critic network's optimal parameters $\phi_q^*$ by minimizing the MSE between estimated $Q$ value of the critic network $Q(s_k, a_k\vert \phi_q)$ and the $Q$ target $y_k$
\begin{subequations}
\label{main-critic}
	\begin{alignat} {3}
		& \phi_q^* = 
		&	 & \arg\min_{\phi_q}~\mathcal{L}(\phi_q), \\
		& \text{s.t.}
		&   &   \mathcal{L}\left(\phi_q\right)=\frac{1}{N_B} \sum_{k=1}^{N_B}\left(y_k-Q\left(s_k, a_k \vert \phi_q\right)\right)^2. \notag 
    \end{alignat}
\end{subequations}
\subsubsection{Training the Actor Network}
To improve the policy obtained from the proposed algorithm, the parameters of the actor network  must be updated so that the action output from the network is in the direction of increasing the Q value, which is the output from the critic network. By using the gradient descent algorithm, the optimal parameters of the actor network $\phi_\mu^*$ can be updated via the following optimization function.
\begin{subequations}
\label{main-actor}
	\begin{alignat} {3}
		& \phi_{\mu}^* = 
		&	 & \arg\max_{\phi_\mu}~\mathcal{L}(\phi_{\mu}), \\
		& \text{s.t.}
		&   &   \mathcal{L}\left(\phi_\mu\right)=\frac{1}{N_B} \sum_{k=1}^{N_B} Q\left(s_k, \pi(s_{k}\vert \phi_\mu) \vert \phi_q\right). \notag
    \end{alignat}
\end{subequations}

\subsubsection{Slow Update of Critic and Actor Networks} 
Continuous updates of the critic and actor networks are necessary in DRL to enhance system performance. \textcolor{duong}{However, excessively frequent updates can result in unstable and erratic fluctuations of the model parameters. This is because the state observations $s_k$ vary widely with each episode $k$ (due to the different environment initialization at the start of every episode), causing the learned trajectory to diverge significantly from those in other episodes. \cite{2016-DRL-FastSlow}}. This instability can lead to overfitting, restricting the model's ability to generalize to novel scenarios beyond the training dataset. 

To address the aforementioned issues, the target actor and critic networks are employed to update the actor and critic networks at a slower pace. This approach ensures stability and maintains a balance between exploration and exploitation throughout the learning process. The critic is accompanied by a slower updated target critic. The target critic is a replica of the critic, but its parameters $\phi_{q_t}$ are updated gradually using a weighted average of the critic's parameters $\phi_q$. Specifically, parameters set of the target critic is computed by taking a fraction (commonly denoted as $\kappa$, with values like $0.001$ or $0.005$) of its current value and adding $(1-\kappa)$ times the new value from the main actor network. The parameters of the target critic network are updated according to the following function:
\begin{align}
    \phi_{q_t}=\kappa \phi_q+(1-\kappa) \phi_{q_t}, \quad 0<\kappa< 1.
\end{align}
Similar to the slow update for target critic, we implement slow updates in the target actor. This technique helps minimize anomalous updates, prevents overfitting, and ultimately improves the overall learning process of the system. The parameters of the target actor are updated as follows:
\begin{align}
\phi_{\mu_t}=\kappa \phi_\mu+(1-\kappa) \phi_{\mu_t}, \quad 0<\kappa < 1.
\end{align}
\subsubsection{OU Noise} 
From the user's perspective, the actor network is of paramount importance as it provides the desired solution. To encourage exploration of the environment, the output of the actor network is subject to the introduction of noise. The DDPG algorithm employs OU Noise \cite{OUNoise} for this purpose due to its two key characteristics: (1) it enables intensive exploration during the initial phase of DRL training, and (2) gradually transitions towards exploitation (i.e., reducing the noise) in the later stages of training.      
\subsubsection{Experienced Replay Buffer} 
In order for the DRL agent based on the DDPG algorithm to learn from the interaction with the environment, an experience replay $D$ is constructed. At time $t$, the DRL agent takes action $a_t$ under the influence of $s_t$, then receives a reward $r_t$ and is moved to the next state. Therefore, the transition tuple ($s_t$, $a_t$, $r_t$, $s_{t+1}$) is saved to experience replay $D$.

\subsection{Closed-form Expression for Optimization Problem} \label{sec:closed-form}
\textcolor{duong}{To reduce the complexity of the optimization problem~\eqref{opt-pro-training:opt-pro} when applying in DDPG}, our target is to find the closed form of the optimization problems. To achieve this, we classify the constraints in the optimization problem into two categories, subbed \emph{explicit constraints} and \emph{ambiguous constraints}. 
\subsubsection{Explicit Constraints}
\label{sec:explicit-constraint}
The \emph{explicit constraints} refers to constraints that are easily configurable. These constraints mostly focus on the limits of system variables (e.g., power, bandwidth). Taking advantage of the outputs of a deep model can lead to substantial improvements in these components without incurring significant computational costs. More specifically, we utilize activation functions to control the output of the training model, ensuring that the control variables will satisfy the constraints. (i) To restrict the variables to the specific range (e.g., $0\leq P \leq P_\textrm{max}$), we use the sigmoid function. Subsequently, the action is normalized to the appropriate range using the normalization function. (ii) To impose a limit on the value while ensuring that the variables always sum to a value lower than a predetermined threshold (e.g., $0\leq \sum_u B_u \leq B$), we make use of the softmax function. Following that, we employ the normalization process, which is similar to the one described in (i). 

This approach can simplify the loss function and improve the performance of the DRL algorithm. Instead of incorporating constraints directly into the reward function, the constraints are hard-fixed, automatically restricting variables to the specific bounds. This simplifies the loss function, avoiding the complexities introduced by constraints, and avoiding the stronger non-convexity of the loss function. Moreover, this approach enhances the model's adaptability, allowing it to quickly learn and adapt to new constraints. By reducing the computational complexity and increasing the adaptability of the system, we can optimize performance without compromising the result quality. Furthermore, this approach can greatly enhance the performance of the DRL algorithm by allowing the model to learn and adapt to new constraints quickly.
\subsubsection{Ambiguous Constraints}
Contrary to explicit constraints, the \emph{ambiguous constraints} are more challenging to configure and cannot be simply addressed by normalizing the model's output. We use the Lagrange approach \cite{2004-CVX-Optimization} to resolve this issue.
\begin{figure}[t]
\centering
\includegraphics[width =1\linewidth]{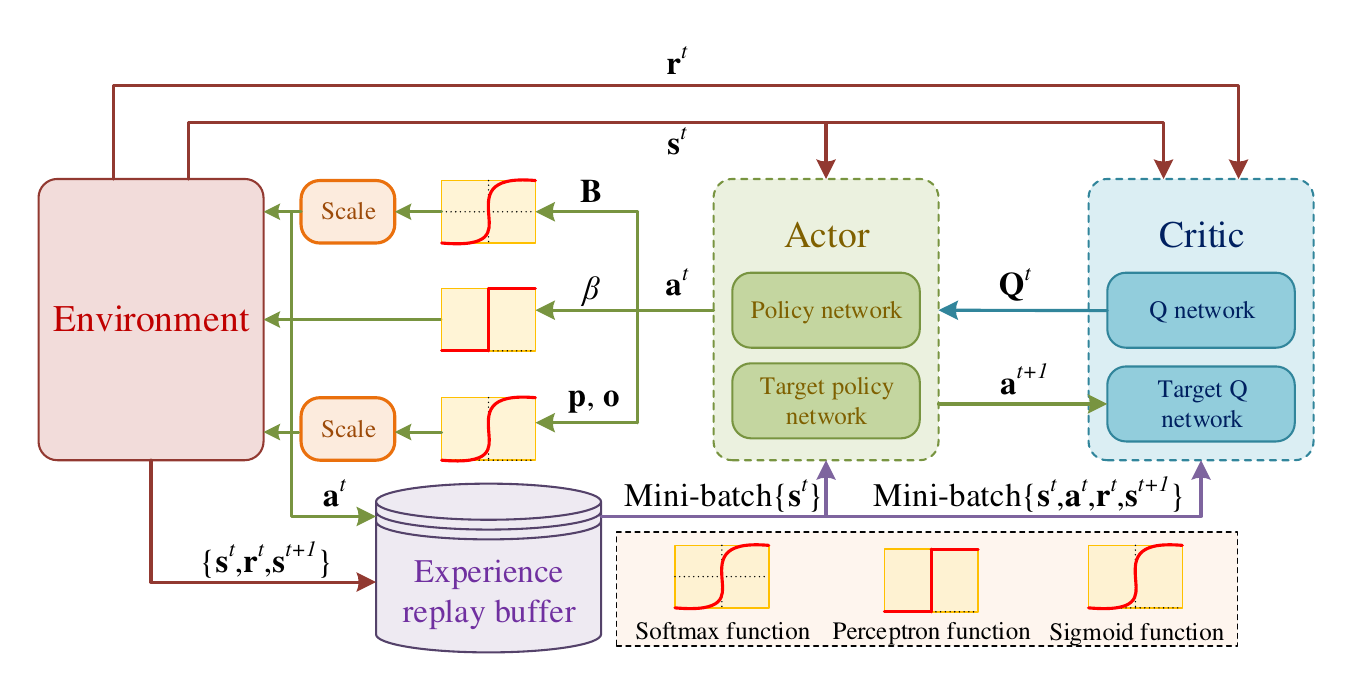}
\caption{\textcolor{duong}{The key features of DDPG-EI are the action settings: the actions belong to \emph{explicit constraints} from \textit{actor network} are scaled using \textit{Sigmoid}, \textit{Softmax} or \textit{ReLU} functions according to specific \emph{explicit constraints} for computational complexity reduction.}}
\label{fig:DDPG-EI}
\end{figure}
\subsection{Proposed Approach DDPG-EI}
\textcolor{duong}{From the proposed \emph{explicit constraints} and \emph{ambiguous constraints}, we redesign the DDPG algorithm into our new proposed DDPG-EI (see Fig.~\ref{fig:DDPG-EI}). Our main focus is to define action spaces to manage the data range and hence, satisfy the \emph{explicit constraints}. By doing so, we can reduce the optimization burden on the reward function to the \emph{ambiguous constraints}. The DDPG-EI can be demonstrated as follows:}
\subsubsection{State Space} 
Let $h^t_u$ denote the channel gain with the BS and the user $u$ at the time step $t$.  Then, the state space of the system, denoted by $S$, is defined as:
\begin{align}
    s_u^t = \{ h^t_1, \ldots h^t_U\}.
\end{align}
As a consequence, the state space $s_u$ can be defined as $s_u = \{s_u^1, \ldots, s_u^T \}$, where $T$ is the total number of time steps used for the RL training phase.
\subsubsection{Action Space}
\textcolor{duong}{Our objective is to design the action output with \emph{explicit constraints} as proposed in Section~\ref{sec:explicit-constraint}.} Given a certain state s, a control action is performed to determine $\mathbf{\beta} = \{ \beta_1, \beta_2, \ldots \beta_U\}$ for user selection, $\mathbf{P} = \{P_1, P_2, \ldots P_U\}$ for power control and $\mathbf{o} = \{ o_1, o_2, \ldots, o_U\}$ for compression ratios. \textcolor{duong}{Denote $\mathbf{A}$ as the action space of the system, $\mathbf{A}$ can be defined as: $\mathbf{A} = \{ \beta_1, P_1, o_1,\ldots \beta_U,P_U,o_U\}$.}
\subsubsection{Reward Function}
As mentioned in Section~\ref{sec:problem-formulation}, we minimize the total communication energy while maintaining the AI performance according to different tasks (i.e., AI model during the execution and training process). Due to the requirement of minimizing total energy over time, we use the total communication energy as the reward. Furthermore, we consider the data distortion constraints as a penalty in the reward function. Therefore, the reward can be defined as $\mathcal{R}=\{r^1,\ldots,r^T\}$, where the immediate reward $r^t$ is defined as:
\begin{align}
r^t= - \sum\nolimits_{u=1}^U E_{\textrm{sem}}^u + \lambda \mathcal{P}, 
\label{eq:temporal-reward}
\end{align}
where the subscript $\lambda$ is the coefficient that controls the goal-oriented regularization $\mathcal{P}$. Specifically, $\mathcal{P}$ is defined as the constraints on \emph{distortion resilience} according to different agent's settings. The core concept is that if a distortion resilience value surpasses a predetermined threshold, the AI model will focus on minimizing the deviation from that threshold. In detail, we define the distortion resilience penalty as:
\begin{align}
    \mathcal{P} = \max{\left\langle\left[\left(\eta^2 \frac{L}{2}-\eta\right)\left(L \sigma_{\mathrm{tot}}\right)^2-\varepsilon_\textrm{target}\right],0 \right\rangle},
\end{align}
and the penalty function for inference resilience task as
\begin{align}
    \mathcal{P} = \max{\left\langle\left[\frac{1}{\sqrt{2\pi}\sigma} p(\mathbf{t}\vert s_u) \exp{\frac{-W^2}{2\sigma^2_\textrm{tot}}}-\varepsilon_\textrm{target}\right],0 \right\rangle},
\end{align}
where the $\max(\cdot)$ function imitates the resilience property, which only captures the penalty only when the data distortion passes the upper threshold (i.e., where the AI task's performance is preserved). As it can easily be seen from the equation~\eqref{eq:temporal-reward}, we only consider the \emph{distortion resilience} as a penalty because other constraints \eqref{subeqn-opt-pro-training:compression-ratio}, \eqref{subeqn-opt-pro-training:max-bandwidth-constraint}, \eqref{subeqn-opt-pro-training:min-bandwidth-constraint}, \eqref{subeqn-opt-pro-training:min-power-constraint}, \eqref{subeqn-opt-pro-training:resource-allocation} can be satisfied by tuning the action output as mentioned in Section~\ref{sec:closed-form}. Given the immediate reward defined as in Equation~\eqref{eq:temporal-reward}, we have the accumulative long-term reward for the system, which can be expressed as follows: 
\begin{align}
    \mathcal{R}(\pi) = \sum^{I_\textrm{glob}}_{t} \gamma^{I_\textrm{glob}-t}\mathbf{r}^t (s^t,\pi(s^t)),
\label{eq:accumulative-reward}
\end{align}
where $\mathcal{R}(\pi)$ is the accumulative long-term reward of the agent under policy $\pi(\cdot)$ and $\mathbf{r}^t(s^t, \pi(s^t))$ is the instant reward at time step $t$. The notation $\gamma$ denotes the discount factor for the reward that reflects how much the reward depends on the past performance (i.e., when $\gamma$ is near $0$, the policy evaluation ignores the historical performance, and vice versa).

\subsection{Performance Analysis of DDPG-EI}
\textcolor{duong}{The reward in equation~\eqref{eq:temporal-reward} only considers the \emph{ambiguous constraints} while neglecting the \emph{explicit constraints}. Despite this simplification, the learning process of DDPG-EI can achieve better convergence compared to traditional DDPG. To support our theoretical analysis of the DDPG-EI approach, we have defined the boundaries of the action components within the reward section of the DDPG-EI observations as follows:
\begin{definition}\label{def:ECS_x}
    Given $P_{\mathrm{ECS},i}$ as the penalty for the action's values $x_i$ according to \say{explicit constraints} and come over the boundary:
    \begin{align}
    P_{\mathrm{ECS},i} &= \abs{\max\left\langle x_i - x_i^{\mathrm{max}}, 0 \right\rangle} \notag \\
                       &+ \abs{\max\left\langle x_i^{\mathrm{min}} -x_i, 0 \right\rangle} \geq 0,
    \end{align}
    where $i$ represents the index of different \say{explicit constraints}.
\label{def:ECS-P}
\end{definition}
From Definition~\ref{def:ECS-P}, we have the following lemma.
\begin{lemma}
    Given the EI design, the action value $x_i$ always satisfies $x_i^{\mathrm{min}} < x_i \leq x^{\mathrm{max}},~\forall i$, so that penalty on \say{explicit constraints} $P_{\mathrm{ECS-EI},i} = 0,~ \forall x$. 
\label{lemma:ECS-P}
\end{lemma}
\textit{Proof}. The proof is demonstrated in Appendix~\ref{appendix:ECS_x}.
\begin{theorem}
Any DRL approach with the integrated EI design (DRL-EI) consistently yields higher optimal rewards compared to DRL approach without EI design. In other words,
    \begin{align}
        \mathbf{r}^*_\mathrm{DRL-EI}(\mathbf{s}^t, \mathbf{a}^t) \geq \mathbf{r}^*_\mathrm{DRL}(\mathbf{s}^t, \mathbf{a}^t),~\forall \mathbf{s} \in \mathbf{S},~\forall \mathbf{a} \in \mathbf{A},
    \end{align}
where the $\mathbf{r}^*_\mathrm{DRL-EI}$ and $\mathbf{r}^*_\mathrm{DRL}$ represent the optimal reward of the DRL approaches with and without the integration of EI design, respectively.
\label{theorem:ECS_reward}
\end{theorem}
\textit{Proof}. The proof is demonstrated in Appendix~\ref{appendix:ECS_reward}. 
\\
Theorem~\ref{theorem:ECS_reward} demonstrates that the EI design enhances the learning performance of any DRL approach. Consequently, with the application of DDPG-EI, we can consistently attain superior performance compared to the conventional DDPG method.}

\section{Experimental Evaluations} \label{sec:experimental-results}
We assess DRGO performance in two different schemes: AI inference and AI training. Detailed settings for our experimental evaluations are provided in Appendix~\ref{sec:settings}. The official implementation is available on Github\footnote[1]{\url{https://github.com/Skyd-Semantic/DRGO-SemCom}} . 
\subsection{DRL Training}
\begin{figure}[!htb]
    \centering
    \includegraphics[width=\linewidth]{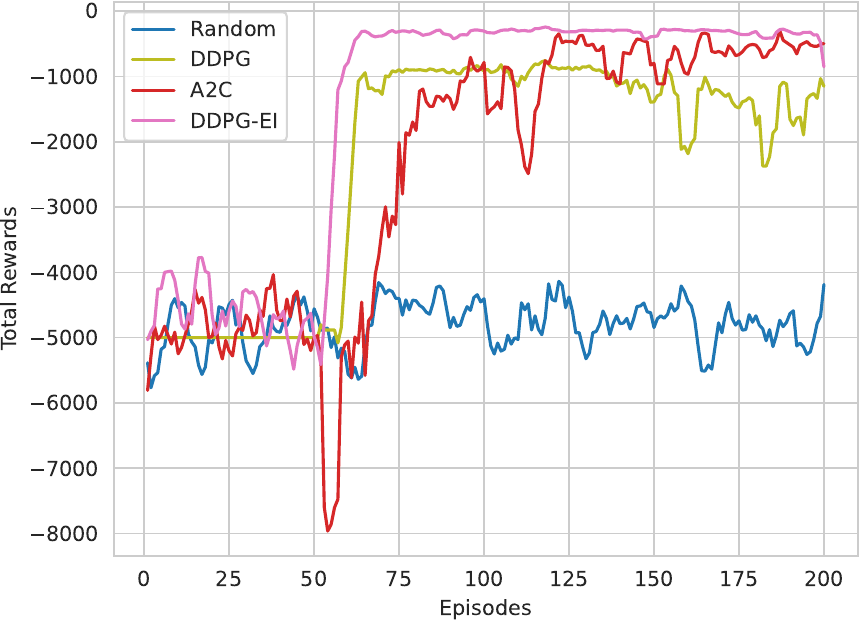}
    \caption{DRL performance on training optimization problem.}
    \label{fig:drl-result}
\end{figure}
We perform an experimental evaluation on the DDPG-EI algorithm and present the results in Figure~\ref{fig:drl-result}. To assess the algorithm's quality, we compare it against various baselines, including Random Action (where actions are selected randomly). Comparing with random action serves to determine if the model enhances system quality beyond the standard level achieved without optimization. Additionally, we compare the proposed solution against DDPG and Asynchronous Actor Critic (DDPG). It is evident that our method inherits the stability of the original DDPG approach and is less affected by noise compared to DDPG. Moreover, by employing fixed constraints at the action outputs, we reduce the unnecessary computational complexity of the optimization problem significantly. As a result, DDPG-EI demonstrates relatively rapid convergence compared to the other two algorithms and yields a much improved policy. Episodes from 0 to 50 involve random actions to populate the experience replay memory, which does not contribute to the training process , thereby avoiding computation costs during this phase. With fast convergence (within only 3 episodes * 200 timesteps per episode from the start of training), the system showcases the agent's ability to rapidly learn from the environment. Consequently, the system can adapt and perform swiftly when integrated into a network.

\subsection{AI Inference Phase}
\subsubsection{Distortion Resilience Boundary}
Two graphs in Figures~\ref{fig:inf_pts} and \ref{fig:inf_boundary} illustrate the influence of data distortion on the generalization gap, which is the difference in AI predicting accuracy between original data and the distorted data (accumulated from various forms of distortion during data transmission).
\begin{figure}[!htb]
	\centering
	\subfloat[\label{fig:inf_boundary}]{\includegraphics[width=0.485\linewidth]{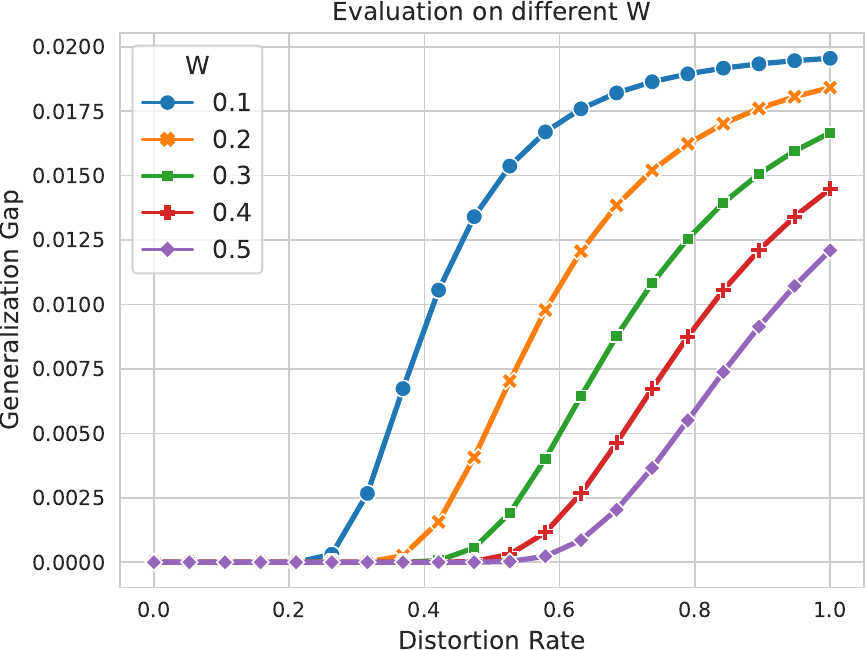}}
    \subfloat[\label{fig:inf_pts}]{\includegraphics[width=0.485\linewidth]{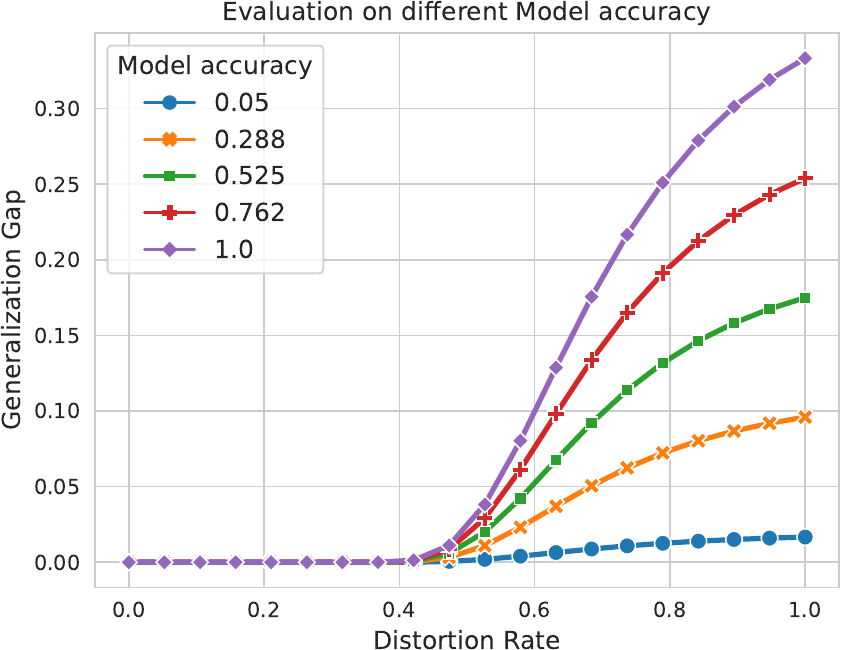}} 
    \caption{Illustration of the inference generalization gap (a) with different decision boundaries, and (b) with different model accuracy.}
\end{figure}
Figures~\ref{fig:inf_pts} and \ref{fig:inf_boundary} present a comparison of AI model performance across various model accuracies and decision boundaries, respectively, quantified by the posterior probability $p(\mathbf{t}\vert \mathbf{\hat{s}})$. As depicted in the graph, higher AI performance (indicated by greater values of $p(\mathbf{t}\vert \mathbf{\hat{s}})$) corresponds to a wider gap in generalization between predictions made using distorted data and clean data, respectively. Another influential factor on predictive performance is the decision boundary $W$. More specifically, as $W$ decreases in size, the generalization gap becomes more prominent. This is because a smaller decision boundary signifies more intricate data characteristics. For instance, referring to Figure~\ref{fig:garnacho}, when the decision boundary is elevated, alterations like changes in player's skin color can significantly impact the accuracy of player identity recognition.

Nevertheless, it is noticeable that the generalization gap experiences substantial impact as the data distortion surpasses a threshold of $0.2$. Consequently, it is evident that the optimal outcomes for our SemCom system are attainable when the distortion rate falls within the range of $0.2$ to $0.4$.

\subsubsection{Different levels of maximum power}
\begin{figure}[!htb]
	\centering
	\subfloat[\label{fig:t-infer-power}]{\includegraphics[width=0.48\linewidth]{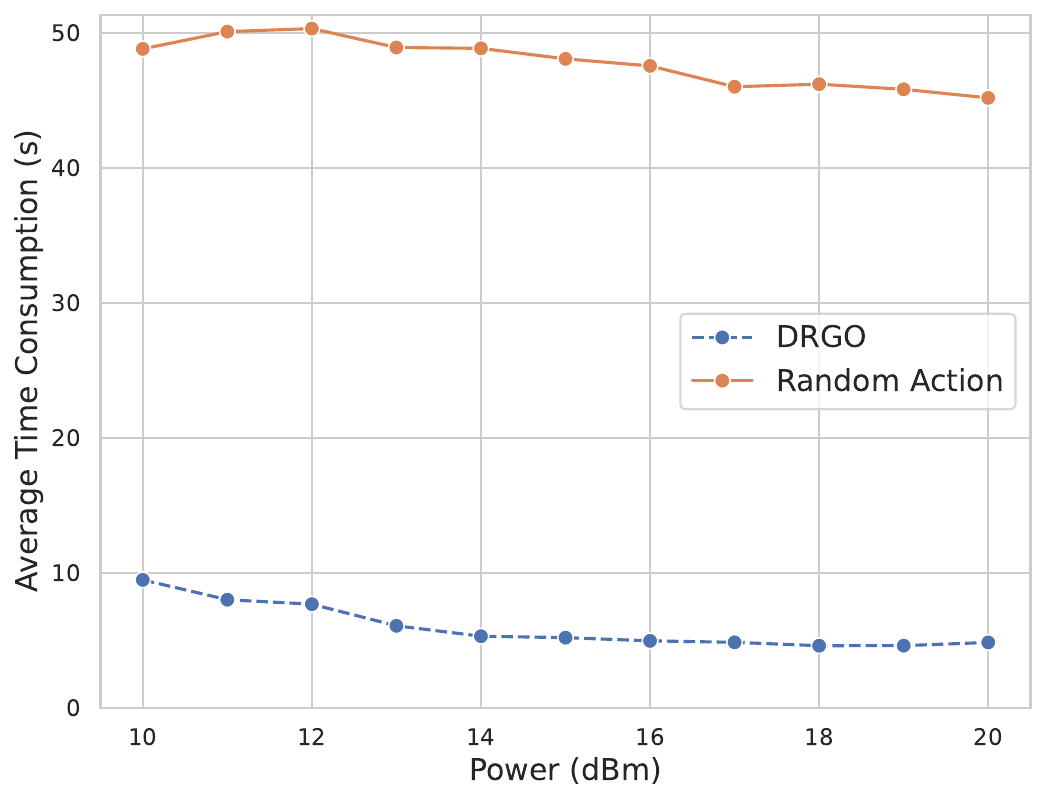}} 
	\subfloat[\label{fig:e-infer-power}]{\includegraphics[width=0.485\linewidth]{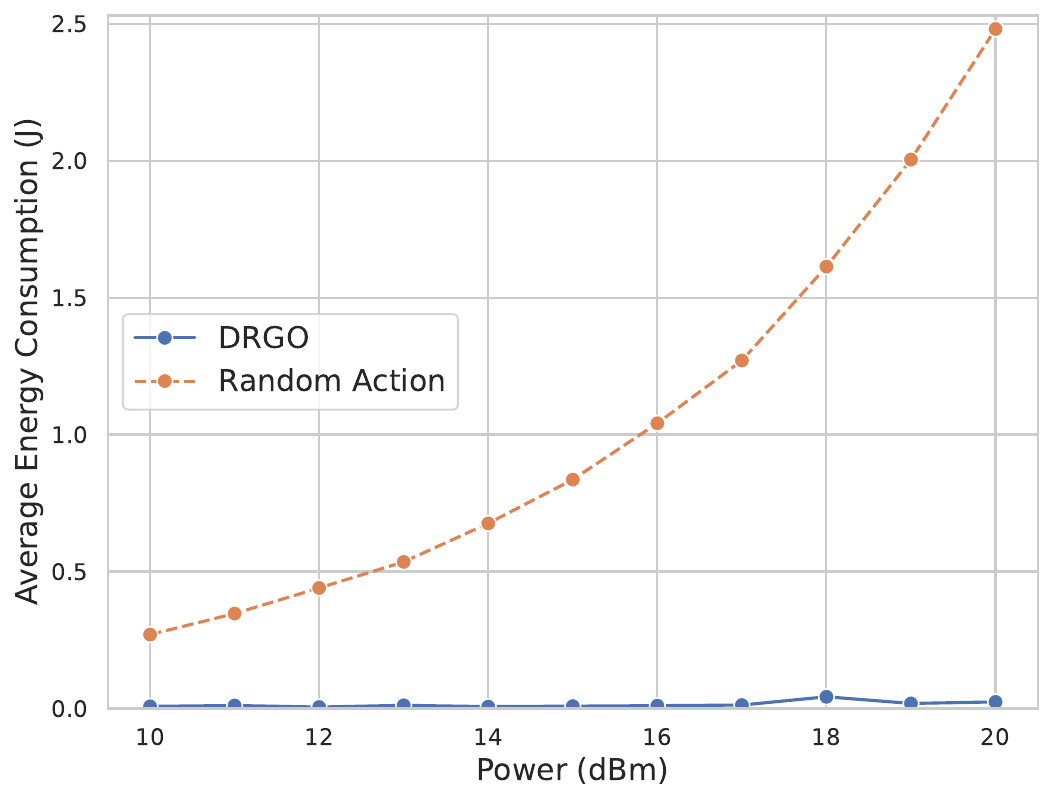}}
    \caption{The average transmission time and energy consumption versus the maximum transmit power of each physical device with $U$ = 10.}
\end{figure}
Figures~\ref{fig:t-infer-power} and \ref{fig:e-infer-power} depict the variations in total time and transmission energy consumption as the maximum transmit power of each user varies. 
Notably, our proposed DRGO system demonstrates a remarkable improvement, achieving a more than fivefold increase in both time delay and energy efficiency. This substantial enhancement can be attributed to the DRGO system's deliberate choice of a high compression ratio, which effectively optimizes resource allocation.

\subsubsection{Different numbers of users}
\begin{figure}[!htb]
	\centering
	\subfloat[\label{fig:t-infer-user}]{\includegraphics[width=0.477\linewidth]{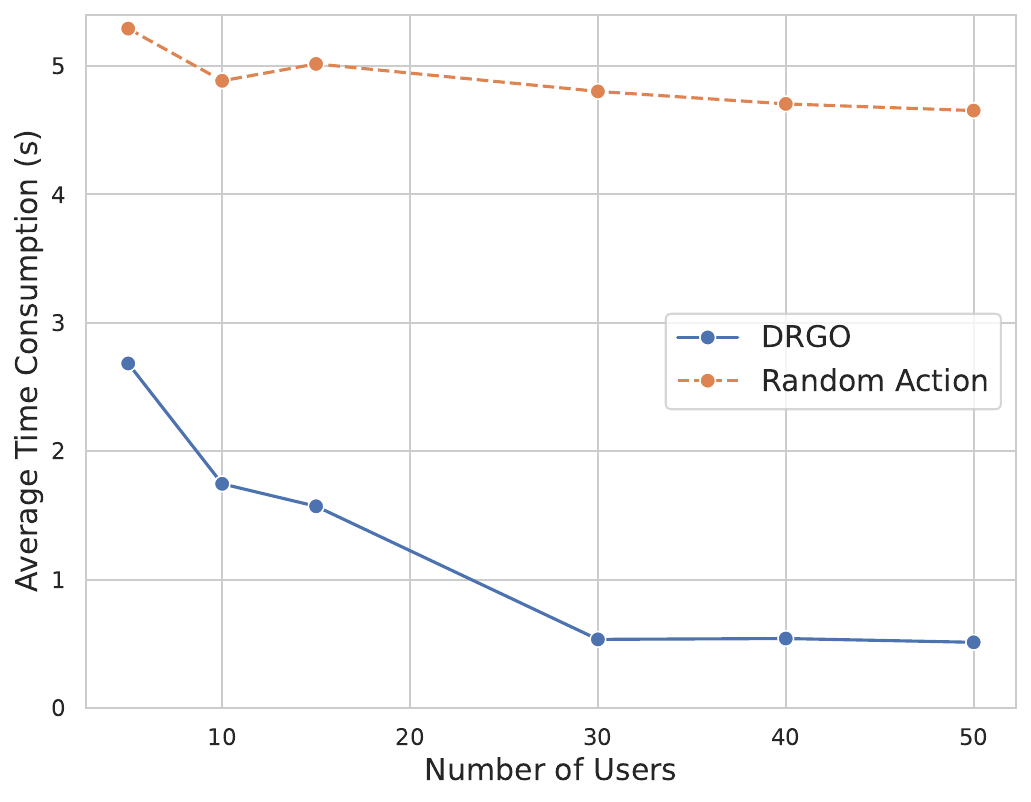}} 
	\subfloat[\label{fig:e-infer-user}]{\includegraphics[width=0.495\linewidth]{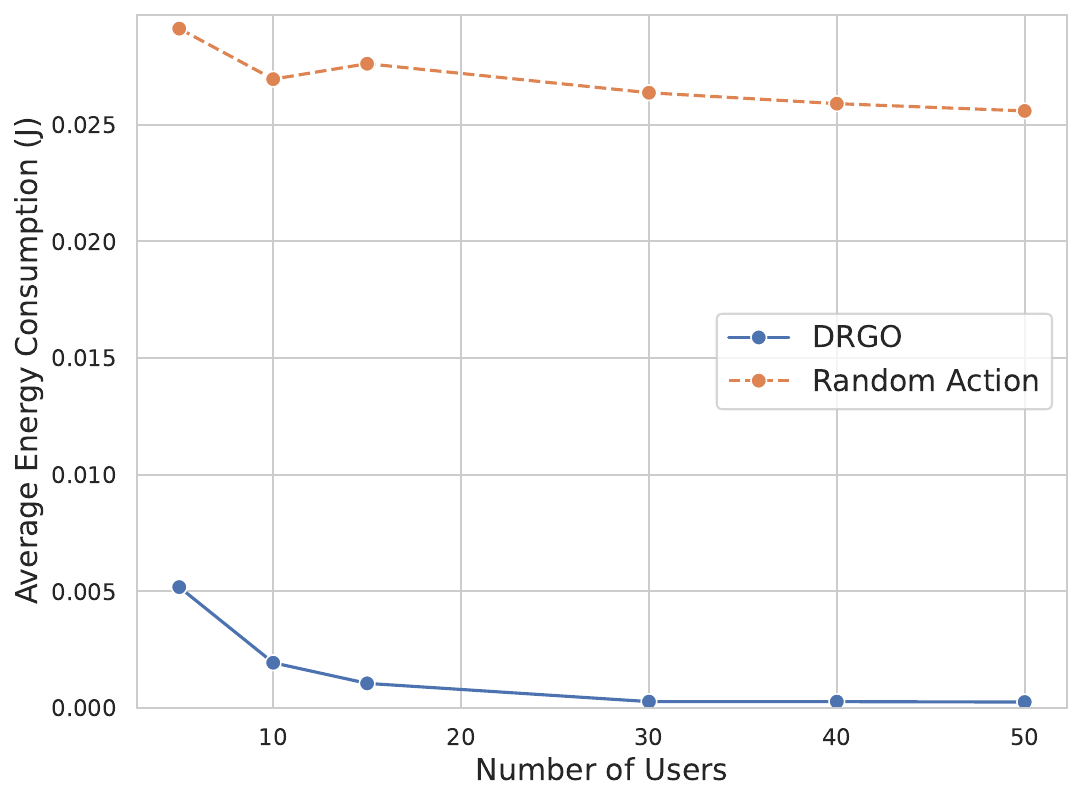}}
    \caption{The average time delay and transmission energy consumption versus the number of participating users with $P_\textrm{max}$ = 10 dBm.}
\end{figure}
Figures~\ref{fig:t-infer-user} and \ref{fig:e-infer-user} illustrate how the average energy and time delay vary with changes in the maximum transmit power of each user. We apply the experiments on different number of users (i.e., $U\in\{5,10,15,20,30,40,50\}$. According to the figures, our proposed DRGO system can achieve a significant improvement in energy efficiency, exceeding a factor of $10$, while maintaining significantly lower time delays compared to random actions (specifically, reducing the time delay from a range of $2$ to $5$ to a much lower level). This remarkable performance is primarily attributed to the utilization of compression techniques. By taking into account semantic distortion, DRGO can proactively assess the level of distortion resilience required. Consequently, the DRGO system can opt for a very high distortion ratio, enabling compression rates exceeding 60 times that of the original data. As a result, users only need to employ very low transmission power when sending data over the wireless channel, typically less than $-3.02$ dB ($0.0005$ W).

\subsection{AI Training Phase}
\subsubsection{Distortion Resilience}
Two graphs illustrate the influence of data distortion on the generalization gap, which is the difference in AI training performance between the original dataset and the dataset subjected to distortion-induced noise (accumulated from various forms of distortion during data transmission).
\begin{figure}[!htb]
	\centering
	\subfloat[\label{fig:train_lr}]{\includegraphics[width=0.485\linewidth]{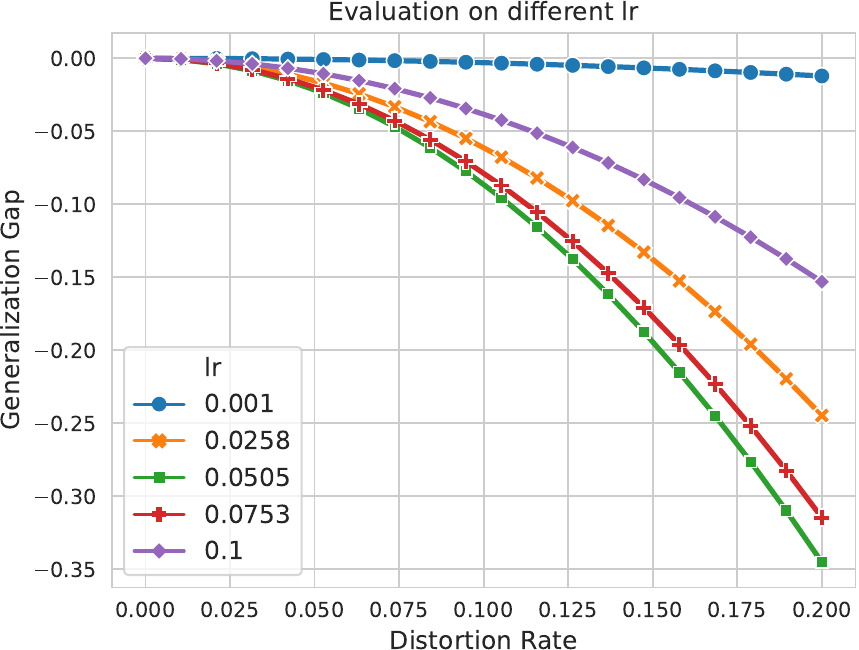}} 
	\subfloat[\label{fig:train-L}]{\includegraphics[width=0.485\linewidth]{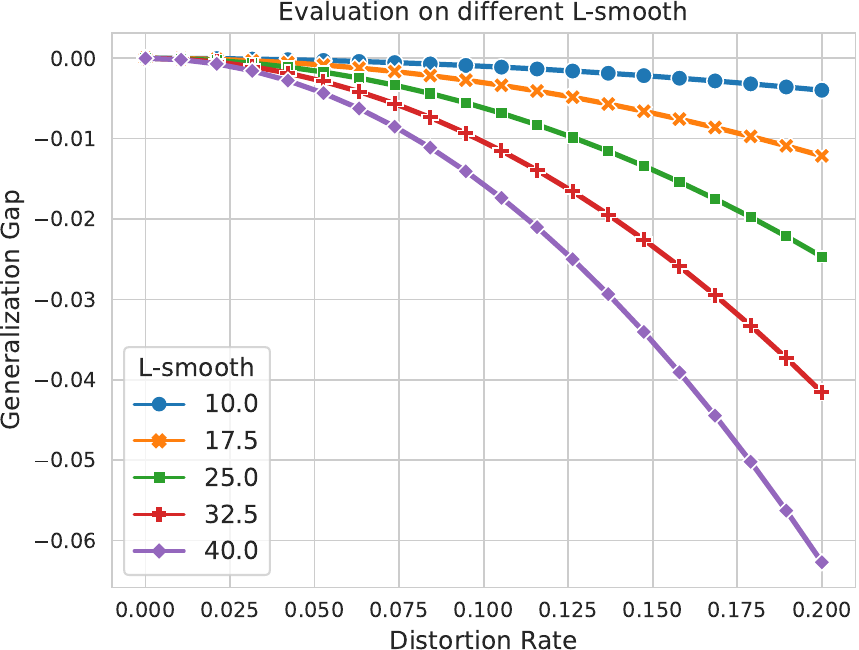}}
    \caption{Illustration of the training gradient inductive bound (a) with different $L$-smooth, and (b) with different learning rate $\eta$.}
\end{figure}
Figure~\ref{fig:train_lr} shows the generalization gap performance with different learning rates. As observed in the graph, when the learning rate is increased, the agent explores more aggressively, causing the gradient to oscillate around the minimizer instead of converging within it. This results in a significant disparity when using distorted data compared to clean data. Conversely, selecting a lower learning rate reduces the discrepancy in noise levels. However, this elongates the learning process and presents risks of gradient descent becoming trapped in sharp minimizers.

Figure~\ref{fig:train-L} evaluates the system performance with different levels of L-smoothness. L-smoothness corresponds to the dataset's complexity. Specifically, datasets with higher complexity result in more sharp minimizers. Therefore, the aforementioned datasets have larger L-smooth values. As seen, higher L-smooth values lead to greater discrepancies in error between the original and distorted data, especially with higher data distortions. This insight indicates that optimizing DRGO also hinges on recognizing variations in dataset complexity . For instance, with complex datasets like ImageNet, employing less distortion resilience is necessary, while simpler datasets call for more resilience. Moreover, one approach to enhancing distortion resilience is to partition the AI task into simpler subtasks and carry out independent multi-task learning. This helps reduce data complexity for each specific task. Several techniques can be applied in this context, such as Disentanglement Learning, Invariant Learning, Multi-task Learning, and Meta Learning.

\subsubsection{Different levels of maximum power}
\begin{figure}[!htb]
	\centering
	\subfloat[\label{fig:t-learn-power}]{\includegraphics[width=0.485\linewidth]{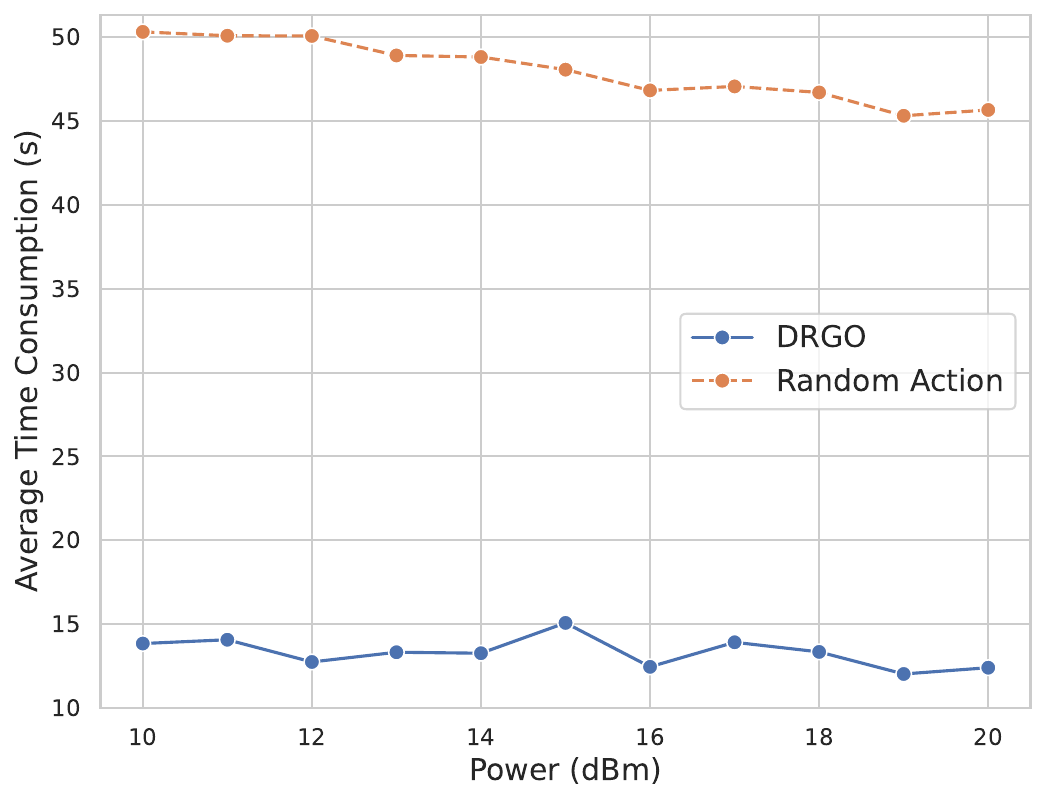}} 
	\subfloat[\label{fig:e-learn-power}]{\includegraphics[width=0.487\linewidth]{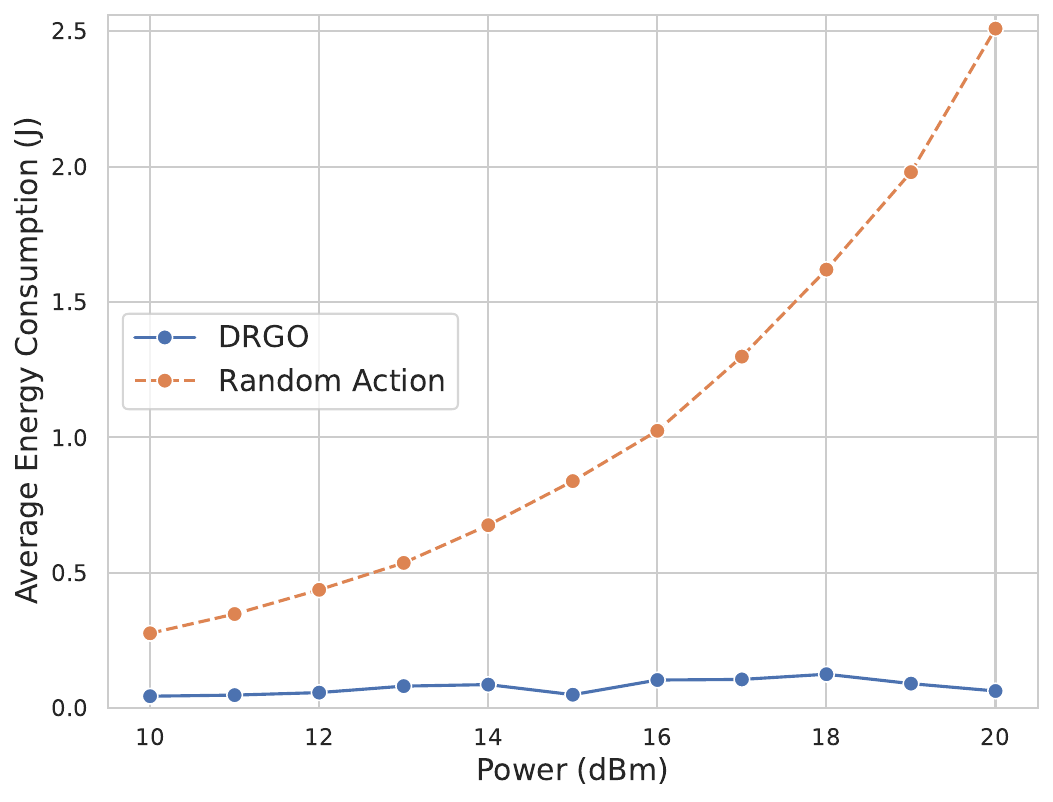}}
    \caption{The average time delay and transmission energy consumption versus the maximum transmit power of each physical device with $U$ = 10.}
\end{figure}
In Figures~\ref{fig:e-learn-power} and \ref{fig:t-learn-power}, we can observe the fluctuations in average energy and time delay as the maximum transmit power of each user undergoes changes. The figure provides a clear depiction of the relationship between increased maximum transmit power at the BS and decreased total transmission energy consumption. This phenomenon arises from the fact that higher transmit power leads to reduced transmission time, allowing more time for computational tasks and ultimately resulting in reduced total transmission energy consumption. Notably, our innovative DRGO system showcases a remarkable enhancement, achieving a more than fivefold improvement in both time delay and energy efficiency. This substantial advancement can be credited to the intentional selection of a high compression ratio by the DRGO system, effectively optimizing the allocation of resources.

\subsubsection{Different numbers of users}
\begin{figure}[!htb]
	\centering
	\subfloat[\label{fig:t-learn-user}]{\includegraphics[width=0.485\linewidth]{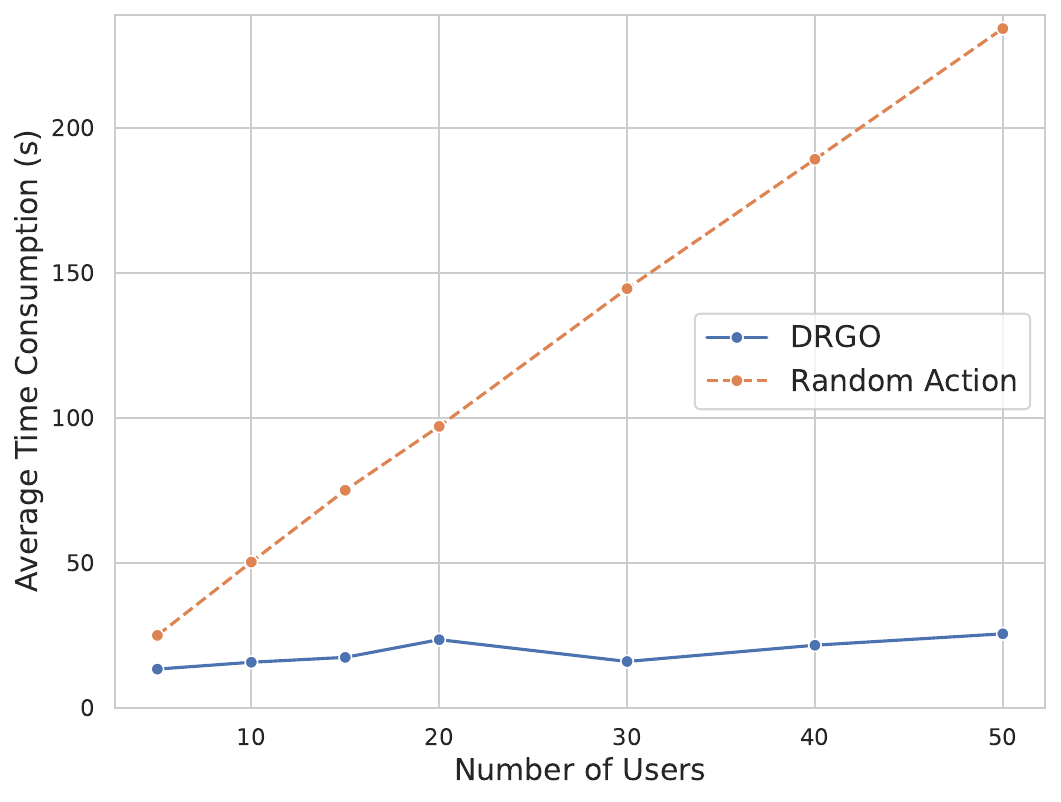}} 
	\subfloat[\label{fig:e-learn-user}]{\includegraphics[width=0.485\linewidth]{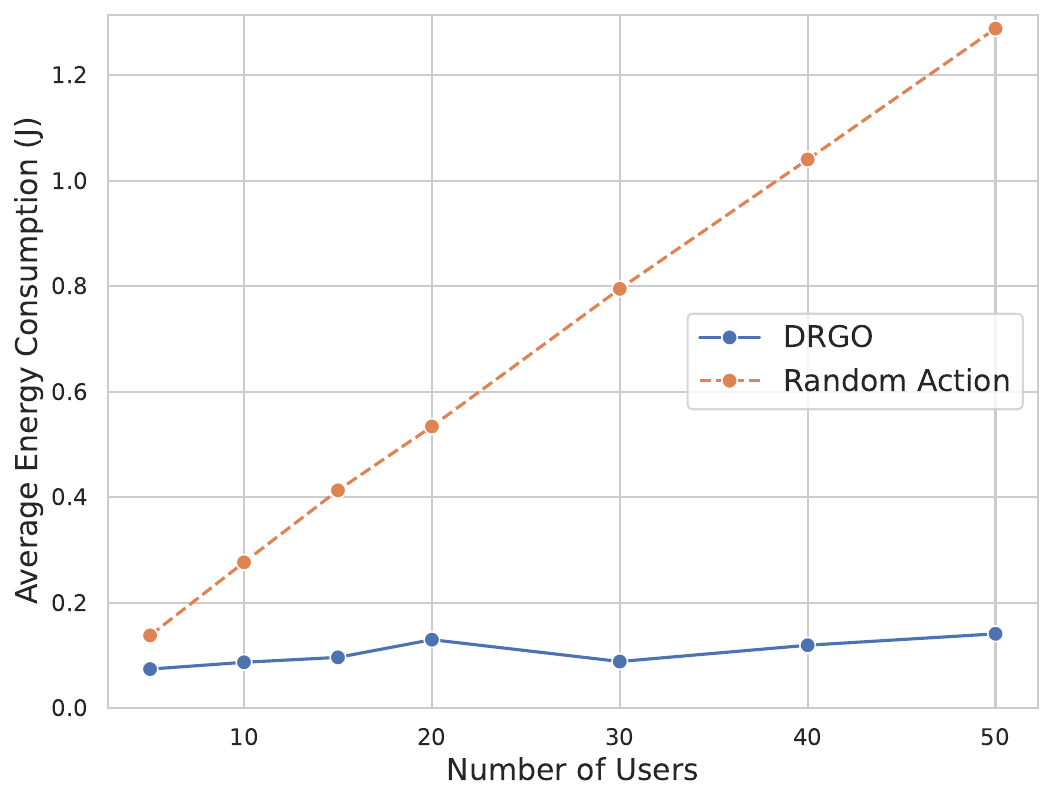}}
    \caption{The average time delay and transmission energy consumption versus the number of participating users with $P_\textrm{max}$ = 10 dBm.}
\end{figure}
Figures~\ref{fig:e-learn-user} and \ref{fig:t-learn-user} provide a visual representation of how average transmission energy consumption and time delay respond to variations in the maximum transmit power of each user. We apply the experiments on various numbers of users (i.e., $U\in\{5,10,15,20,30,40,50\}$. These figures underscore the impressive capabilities of our proposed DRGO system, which achieves a remarkable enhancement in energy efficiency, surpassing a tenfold improvement. Additionally, it maintains significantly reduced time delays compared to random actions, notably shrinking the time delay ranges from $2$ to $5$ to a substantially lower level. The extraordinary performance of our system is predominantly attributed to the strategic use of compression techniques. By factoring in semantic distortion considerations, DRGO can proactively gauge the required level of distortion resilience. Consequently, DRGO can select an exceptionally high distortion ratio, resulting in compression rates exceeding 60 times that of the original data. As a consequent, users only need to employ very low transmit power, typically falling below $-3.02$ dB, when transmitting data over the wireless channel.

\section{Conclusion}
This paper has explored efficient SemCom by addressing wireless resource allocation and semantic information extraction with data distortions. This is the first SemCom network optimization system that can maintain consistent semantic metrics and address the straggling issue of optimization of SemCom system. A novel approach has been introduced to enable users to transmit lossy data to a BS equipped with computing capabilities, which then performs AI operations based on the received data with a predefined distortion. This research has pioneered the integration of semantic distortion considerations into AI training, treating the learning process as an optimization problem to minimize the transmission energy consumption while accommodating generalization gaps. Moreover, this work has also presented an effective DRL solution, demonstrating the effectiveness of the proposed algorithm through numerical results.

\label{sec:conclusion}
\bibliographystyle{IEEEtran}
\bibliography{DRGO.bib}
\vspace{-0.5cm}
\begin{IEEEbiography}[{\includegraphics[width=1in,height=1.25in,clip,keepaspectratio]{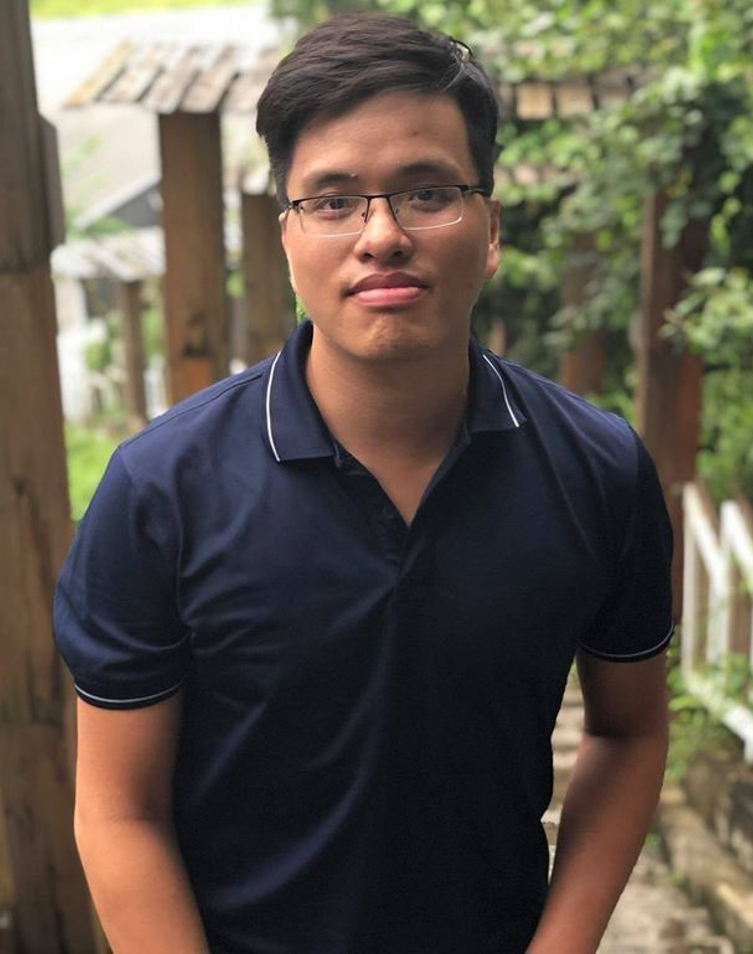}}]{Minh-Duong Nguyen} 
received a B.S. degree in electronics and telecommunications engineering from the Hanoi University of Science and Technology, Vietnam, in 2016. He was a DSP engineer in Viettel R\&D Center from 2017 to 2019. He was working in electronic warfare (i.e., electronic warfare systems and GPS/4G jamming systems).  
He was a senior embedded engineer in Vinsmart, Vingroup, from 2019 to 2020, developing physical layers for the 5G Base Station. He is currently pursuing his Ph.D.'s degree in the department of information convergence engineering at Pusan National University, South Korea. His research interests include reinforcement learning, federated learning, multi-task learning, meta-learning, domain generalization, data representation, and semantic communication. He was a recipient of the Top Reviewer Award from ICML in 2021, the IEEE ATC Best Paper Award in 2022. 
\end{IEEEbiography}
\vspace{-1cm}
\begin{IEEEbiography}[{\includegraphics[width=1in,height=1.25in,clip,keepaspectratio]{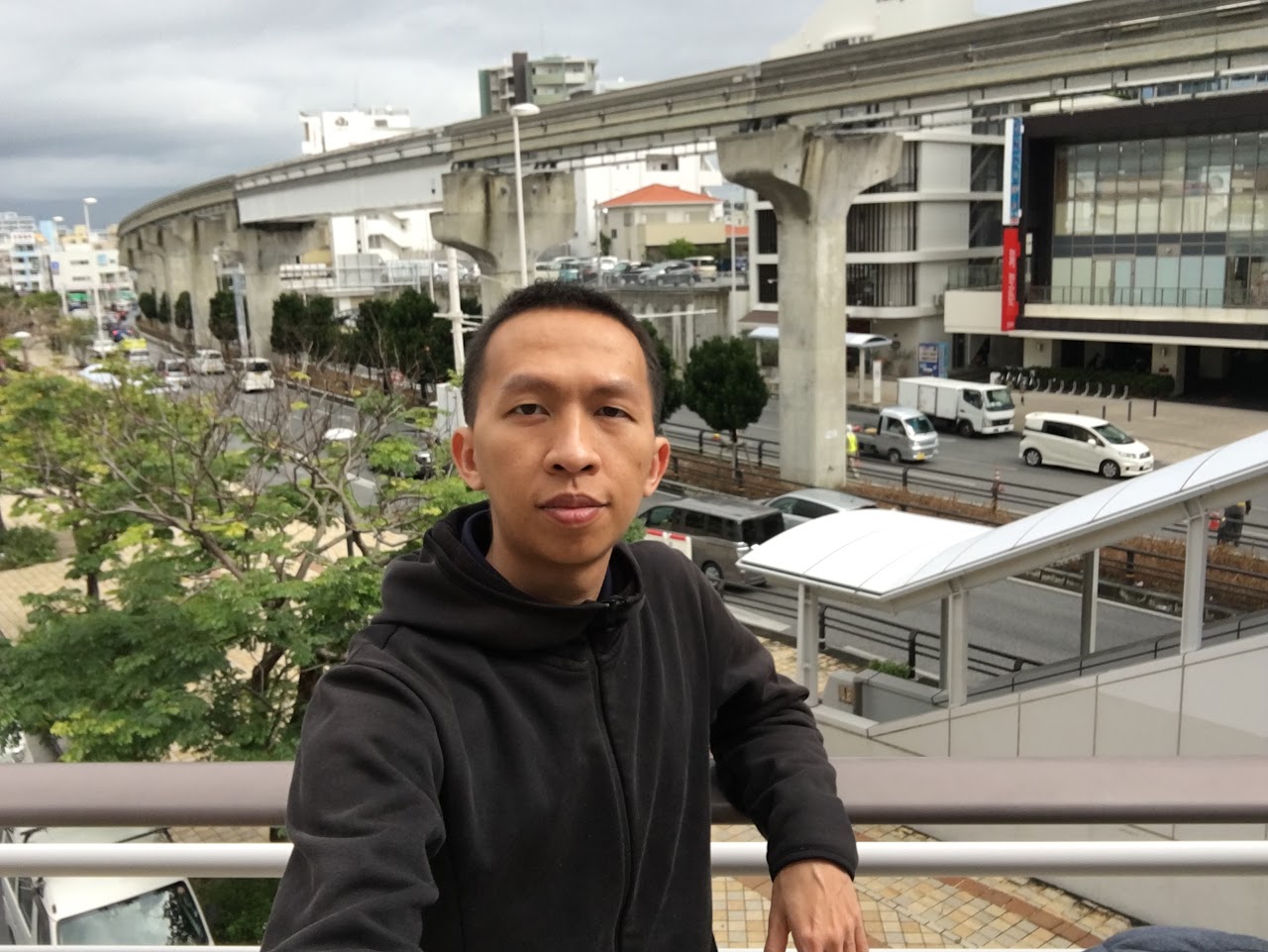}}]{Quang-Vinh Do}
Quang Vinh Do received his MSc degree in Electrical Engineering from Ho Chi Minh City University of Technology, Vietnam, in 2009, his Master’s degree in Electronic and Computer Engineering from RMIT University, Australia, in 2013, and his Ph.D. in Electrical Engineering from the University of Ulsan, South Korea, in 2020. He was a postdoctoral researcher at the University of Ulsan, South Korea, from Sept. 2020 to Feb. 2021. From Mar. 2021 to Aug. 2023, he was a postdoctoral research fellow at the Artificial Intelligence Research Center, Pusan National University, South Korea. He is currently a lecturer and researcher at Faculty of Electrical and Electronics Engineering, Ton Duc Thang University, Ho Chi Minh City, Vietnam. His research interests include developing and applying artificial intelligence techniques to wireless communication networks.
\end{IEEEbiography}  
\vspace{-1cm}
\begin{IEEEbiography}[{\includegraphics[width=1in,height=1.25in,clip,keepaspectratio]{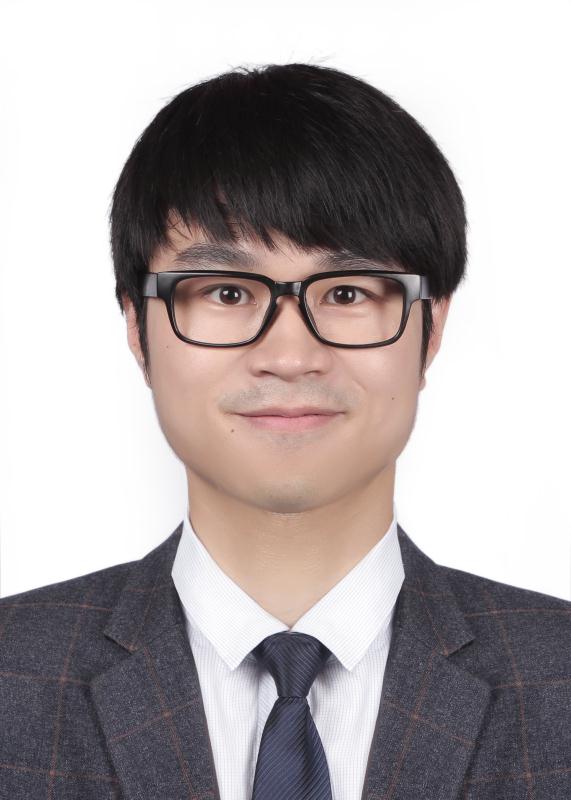}}]{Zhaohui Yang}
(Member, IEEE) received the Ph.D. degree from Southeast University, Nanjing, China, in 2018. He is currently a ZJU Young Professor with Zhejiang Key Laboratory of Information Processing Communication and Networking, College of Information Science and Electronic Engineering, Zhejiang University, Hangzhou, China. From 2018 to 2020, he was a Postdoctoral Research Associate with the Center for Telecommunications Research, Department of Informatics, King’s College London, London, U.K. From 2020 to 2022, he was a Research Fellow with the Department of Electronic and Electrical Engineering, University College London, London. His research interests include joint communication, sensing, and computation, federated learning, and semantic communication. Dr. Yang received the IEEE Communications Society Leonard G. Abraham Prize Award in 2024, the IEEE Marconi Prize Paper Award in 2023, and the IEEE Katherine Johnson Young Author Paper Award in 2023. He currently serves as an Associate Editor for IEEE Transactions on Green Communications and Networking, IEEE Communications Letters, and IEEE Transactions on Machine Learning in Communications and Networking. He has served as a Guest Editor for several journals, including IEEE Journal on Selected Areas in Communications.
\end{IEEEbiography}  
\vspace{-1cm}
\begin{IEEEbiography}[{\includegraphics[width=1in,height=1.25in,clip,keepaspectratio]{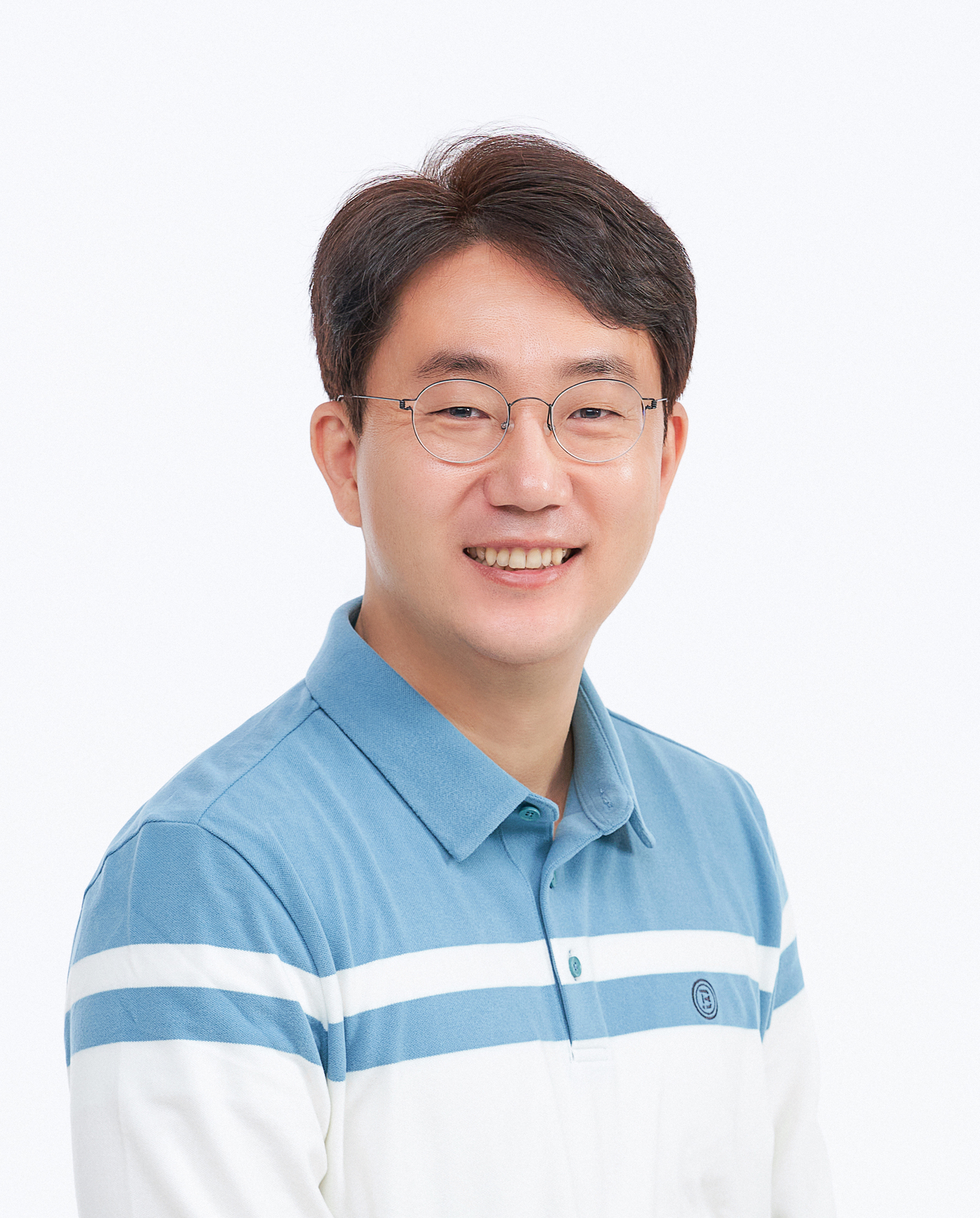}}]{Won-Joo Hwang} 
(S’01-M’03-SM’17) received the B.S. and M.S. degrees in computer engineering from Pusan National University, Busan, South Korea, in 1998 and 2000, respectively. He received the Ph.D. degree in information systems engineering from Osaka University, Osaka, Japan, in 2002. From 2002 to 2019, he was a Full Professor at the Inje University, Gimhae, South Korea. Currently, he is a Full Professor in the Biomedical Convergence Engineering Department at the Pusan National University. His research interests include optimization theory, game theory, machine learning and data science for wireless communications and networking. 
\end{IEEEbiography} 
\vspace{-1cm}
\begin{IEEEbiography}[{\includegraphics[width=1in,height=1.25in,clip,keepaspectratio]{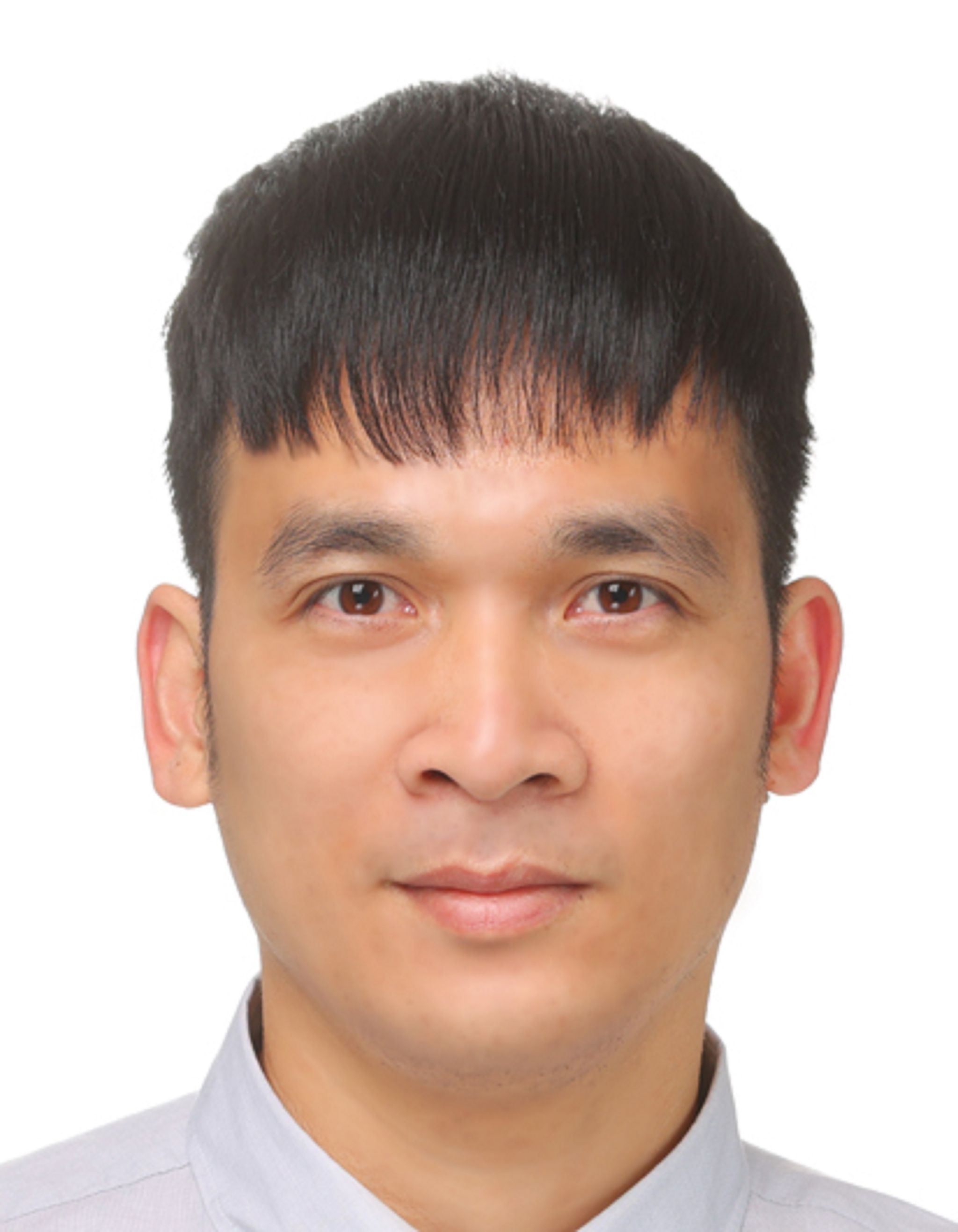}}]{Quoc-Viet Pham} (M’18, SM’23) is currently an Assistant Professor in Networks and Distributed Systems at the School of Computer Science and Statistics, Trinity College Dublin, Ireland. He is also an Associate Investigator of the SFI CONNECT centre and a Supervisor of the SFI ADVANCE centre. He earned his Ph.D. degree in Telecommunications Engineering from Inje University, Korea, in 2017.
He specialises in applying convex optimisation, game theory, and machine learning to analyse and optimise cloud edge computing, wireless networks, and IoT systems. He was awarded the prestigious Korea NRF funding for outstanding young researchers for the term 2019-2024. He was a recipient of the Best Ph.D. Dissertation Award in 2017, the Top Reviewer Award from IEEE Transactions on Vehicular Technology in 2020, the Golden Globe Award in Science and Technology for Younger Researchers in Vietnam in 2021, the IEEE ATC Best Paper Award in 2022, and the IEEE MCE Best Paper Award in 2023. He was honoured with the IEEE ComSoc Best Young Researcher Award for EMEA 2023 in recognition of his research activities for the benefit of the Society. He is serving as an Editor for IEEE Communications Letters, IEEE Communications Standards Magazine, IEEE Communications Surveys \& Tutorials, Journal of Network and Computer Applications, and REV Journal on Electronics and Communications, and has served as an Editor or Guest Editor for IEEE Internet of Things Journal, IEEE Internet of Things Magazine, IEEE Transactions on Consumer Electronics, Computer Communications, and Scientific Reports.
\end{IEEEbiography}   
\vfill

\clearpage
\appendices 
\section{Experimental Settings}
\label{sec:settings}
In this section, we study the performance of the proposed DRGO assisted goal-oriented SemCom by using computer simulation results. All statistical results are averaged over $10$ independent runs.
\subsection{DRL settings}
In our simulations, the learning rates for actor and critic networks are $3\times 10^{-4}$ and $1\times 10^{-3}$, respectively. The reward discount parameter is $\gamma=0.99$, the network soft updating parameter is $\tau = 5\times 10^{-3}$, and the batch size for the experience replay memory is $128$. 

\subsection{System settings}
For our simulations, we deploy $U=50$ users uniformly in a square area of $1\textrm{km}\times 1\textrm{km}$ with the BS located at its center. The path loss model is $128.1 + 37.6 \log_{10} d$ ($d$ is in km) and the standard deviation of shadow fading is $8$ dB. The system bandwidth is configured at $100$ MHz. In addition, the noise power spectral density is $N_0 = -174$ dBm/Hz.

\textcolor{duong}{To simulate the physical layer, we follow \cite{2019-Sem-DJSCCN} and use DJSCC-N (reproducible code is available at Github\footnote[1]{\url{https://github.com/skydvn/SemCom-Pytorch?fbclid=IwZXh0bgNhZW0CMTAAAR0hrOexV_QJqr3Nnz404nX2WEuPpSJx6NAw_7BEU0mMiE_yoRyuKpIsSnQ_aem_BUx-HFJzH8IEUVFWF7wvMg}}) for the channel coding and 16-QAM for the modulation. }

\textcolor{duong}{We use the CIFAR-10 dataset \cite{2010-DL-CIFAR} to evaluate the AI process. The prediction task associated with the data is the classification problem, which consists of $10$ possible labels that needed to be predicted. The images are assumed to be sampled by distributed devices, and transferred via communication link to the BS. We choose an approximate data size of $24.528$ Kb which is related to the averaged size of CIFAR-10 images.}

\subsection{Goal-oriented settings}
To properly simulate the goal-oriented SemCom, we consider the AI training task. In goal-oriented AI task, the key features that capture all characteristics of training data and training AI model are $L$-smooth and $\mu$-strongly convex. To this end, we deploy a classification task on DNN on the CIFAR-10 dataset \cite{2010-DL-CIFAR}. We sample data and feed through the AI model to consider the Hessian of the loss function. We follow the sharp minimizer theory \cite{2017-SharpMinima-Generalize, 2018-DL-LinearScaling, 2020-DL-LAMB} that at the initial phase of the AI training, the minimizer tends to be the sharpest. Therefore, considering the Hessian of the loss function of the untrained AI model with specific dataset, we can have the minimizer with highest second-order derivative value, which is approximately close to the $L$-smooth and $\mu$ strongly convex value. To find the $L$ and $\mu$ values, we implement an experimental evaluation on dataset, and our implementation code is accessible at Github\footnote[2]{\url{https://github.com/Skyd-Semantic/DRGO-SemCom/blob/main/theory_eval/Lsmooth_Estimation.ipynb}.}. It is found that the $L$ and $\mu$ value on MNIST and CIFAR-10 dataset is higher than $10$ and $30$, respectively. Because of the assymetric characteristic of $L$-smoothness and $\mu$ strong convexity, we can choose the $L$ and $\mu$ for our experimental evaluation with the set of values $\{10, 17, 25, 32.5, 40\}$.

\subsection{Semantic Compression settings}
\label{sec:semantic-setting}
To make the appropriate estimation for compression - \say{ratio to data distortion}, we use a convolutional autoencoder on the CIFAR-10 dataset \cite{2010-DL-CIFAR} to simulate the performance of SemCom. To be more specific, we train the autoencoder to reconstruct the original data via the MSE function. The embedding vectors with smallest dimensionality represent the compressed representation and the compression ratio is considered via the equation~\eqref{eq:compression-ratio}. Our empirical implementation for compression - \say{ratio to data distortion} can be found via experimental evaluation code\footnote[2]{\url{https://github.com/Skyd-Semantic/DRGO-SemCom/blob/main/theory_eval/CompressionEvaluation.ipynb}.}. The mapping of \say{ratio to data distortion} is demonstrated as in Table~\ref{tab:semantic-compression}.
\begin{table}[]
\caption{Experimental Results of compressors with different compression ratio.}
\label{tab:semantic-compression}
\begin{tabular}{|c|c|c|c|c|c|}
\hline
Comp.       & Data        & Loss       & Comp.       & Data        & Loss       \\
Ratio       & Dimension   &            & Ratio       & Dimension   &            \\ \hline
192         & 16          & 0.799      & 24          & 128         & 0.193      \\ \hline
96          & 32          & 0.539      & 16          & 192         & 0.131      \\ \hline
48          & 64          & 0.337      & 12          & 256         & 0.098      \\ \hline
32          & 96          & 0.249      & 10          & 312         & 0.079      \\ \hline
9           & 341         & 0.069      & 8           & 384         & 0.060      \\ \hline
6           & 576         & 0.034      & 4           & 768         & 0.020      \\ \hline
3           & 1152        & 0.017      & 2           & 1536        & 0.015      \\ \hline
\end{tabular}
\end{table}

\section{Proof on Lemma~\ref{lemma:sequential-model-distortion}}
\label{appendix:sequential-model-distortion}
Before proving Lemma~\ref{lemma:sequential-model-distortion}, we first adopt the following lemma: 
\begin{lemma}
    Consider the convolution of two consecutive systems which follow the Gaussian process (i.e., $f_1(x) = \mathcal{N}(0, \sigma^2_1)$ and $f_2(x) = \mathcal{N}(0, \sigma^2_2)$). The resulting function follows the Gaussian process and has the resulting variance as $\sigma^2_\textrm{tot} = \sigma^2_1 + \sigma^2_2$. For instance: 
    \begin{align}
        f_\textrm{tot}(x)= f_1(z) \otimes f_2(z) = \frac{1}{\sqrt{2 \pi} \sigma_\textrm{tot}} \exp \left[-\frac{z^2}{2 \sigma_\textrm{tot}^2}\right]
    \end{align}
\label{lemma:conv-gaussian}
\end{lemma}
\textit{Proof:} The proof is demonstrate in Appendix~\ref{appendix:conv-gaussian}.

We consider a sequential process, consisting of $L$ distinct Gaussian process $\mathbb{P}(x^L\vert x^1)$, which can also be considered as a Markov Process. Therefore, the sequential process can be represented as follows: 
\begin{align}
    \mathbb{P}(x^L \vert x^0) = \prod^{L}_{l = 1} \mathbb{P}(x^{l} \vert x^{l-1}).
\end{align}
If we consider each transition probability $\mathbb{P}(x^{l+1} \vert x^l)$ as a function of a transformation from $x^l$ to $x^{l+1}$, we can have the total Markov Process $\mathbb{P}(x^L \vert x^0)$ as a serial convolution function of all component transformation functions. For instance: 
\begin{align}
    \mathbb{P}(x^L \vert x^0) = P(L\Vert L-1)\otimes(\ldots\otimes(P(2\vert 1)\otimes P(1\vert 0))).
\label{eq:serie-conv}
\end{align}
With 2 elements (corresponding to $L=2$), we have:
\begin{align}
    \mathbb{P}(x^2 \vert x^0) = P(x^2 \vert x^1) \otimes P(x^2 \vert x^1).
\end{align}
By applying Lemma~\ref{lemma:conv-gaussian}, the data distortion of the system with two consecutive processes is equal to $\sigma^2_\textrm{tot} = \sigma^2_1 + \sigma^2_2$. By performing the induction operation on \eqref{eq:serie-conv}, we have the total data distortion can be calculated as follows: 
\begin{align}
    \sigma^2_\textrm{tot} = \sum^L_{l=0} \sigma^2_l.
\end{align}

\section{Proof on Lemma~\ref{lemma:conv-gaussian}} \label{appendix:conv-gaussian}
Consider the two convolution of $f_1$ and $f_2$, the resulting distribution $f_\textrm{tot}$ is achieved through the convolution of $f_1$ and $f_2$, expressed as follows:
\begin{align}
    f_\textrm{tot} = f_1(z) \otimes f_2(z)=\int_{-\infty}^{\infty} f_2(z-x) f_1(x) dx,
\end{align}
where $\otimes$ denotes the convolution operation. Given that $f_1$ and $f_2$ are normal densities with two distortion $\sigma_1^2$ and $\sigma_2^2$, respectively. We have the following: 
\begin{align}
f_1(z)=\mathcal{N}\left(0, \sigma_1^2\right)=\frac{1}{\sqrt{2 \pi} \sigma_1} e^{-z^2 /\left(2 \sigma_1^2\right)} \\
f_2(z)=\mathcal{N}\left(0, \sigma_2^2\right)=\frac{1}{\sqrt{2 \pi} \sigma_2} e^{-z^2 /\left(2 \sigma_2^2\right)}.
\end{align}
Substituting these expressions into the convolution equation:
\begin{align}
f_\textrm{tot}(z) 
&=\int_{-\infty}^{\infty} \frac{1}{\sqrt{2 \pi} \sigma_2} \exp \left[-\frac{\left(z-x\right)^2}{2 \sigma_2^2}\right] \times\notag \\
&~~~~~~~~~~\frac{1}{\sqrt{2 \pi} \sigma_1} \exp \left[-\frac{x^2}{2 \sigma_1^2}\right] d x \\
& =\int_{-\infty}^{\infty} \frac{1}{\sqrt{2 \pi} \sqrt{2 \pi} \sigma_1 \sigma_2} \times\notag \\
&~~~~~~~~~~\exp \left[-\frac{\sigma_1^2\left(z-x\right)^2+\sigma_2^2 x^2}{2 \sigma_1^2 \sigma_2^2}\right] d x \\
& =\int_{-\infty}^{\infty} \frac{1}{\sqrt{2 \pi} \sqrt{2 \pi} \sigma_1 \sigma_2} \times  \\
&~~~~~~~~~~\exp \left[-\frac{\sigma_1^2\left(z^2+x^2-2xz\right)+\sigma_2^2 x^2}{2 \sigma_2^2 \sigma_1^2}\right] dx \notag \\
& =\int_{-\infty}^{\infty} \frac{1}{\sqrt{2 \pi} \sqrt{2 \pi} \sigma_1 \sigma_2} \times \\
&~~~~~~~~~~\exp \left[-\frac{x^2\left(\sigma_1^2+\sigma_2^2\right)-2 x \sigma_1^2 z +\sigma_1^2 z^2}{2 \sigma_2^2 \sigma_1^2}\right] d x \notag .
\end{align}
We denote $\sigma_\textrm{tot}=\sqrt{\sigma_1^2+\sigma_2^2}$. Thus, we have the alternative equation as follows:
\begin{align}
& f_\textrm{tot}(z)=\int_{-\infty}^{\infty} \frac{1}{\sqrt{2 \pi} \sigma_\textrm{tot}} \frac{1}{\sqrt{2 \pi} \frac{\sigma_1 \sigma_2}{\sigma_\textrm{tot}}} \times \notag \\
&~~~~~~~~~~~~~~~~~\exp\left[-\frac{x^2 - 2x\frac{\sigma_1^2 z}{\sigma_\textrm{tot}^2}+\frac{\sigma_1^2 z^2}{\sigma_\textrm{tot}^2}}{2\left(\frac{\sigma_1 \sigma_2}{\sigma_\textrm{tot}}\right)^2}\right] dx  \\
& =\int_{-\infty}^{\infty} \frac{1}{\sqrt{2 \pi} \sigma_\textrm{tot}} \frac{1}{\sqrt{2 \pi} \frac{\sigma_1 \sigma_2}{\sigma_\textrm{tot}}} \times \notag \\
&~~~~~~~~~~~~~~~~~\exp \left[-\frac{\left(x-\frac{\sigma_1^2 z}{\sigma_\textrm{tot}^2}\right)^2-\left(\frac{\sigma_1^2 z}{\sigma_\textrm{tot}^2}\right)^2+\frac{\sigma_1^2 z^2}{\sigma_\textrm{tot}^2}}{2\left(\frac{\sigma_1 \sigma_2}{\sigma_\textrm{tot}}\right)^2}\right] d x \\
& =\int_{-\infty}^{\infty} \frac{1}{\sqrt{2 \pi} \sigma_\textrm{tot}} 
\exp \left[-\frac{\sigma_\textrm{tot}^2 \sigma_1^2 z^2-\left(\sigma_1^2 z\right)^2}{2 \sigma_\textrm{tot}^2\left(\sigma_1 \sigma_2\right)^2}\right] \notag \\
&~~~~~~~~~~\frac{1}{\sqrt{2 \pi} 
\frac{\sigma_1 \sigma_2}{\sigma_\textrm{tot}}} 
\exp \left[-\frac{\left(x-\frac{\sigma_1^2 z}{\sigma_\textrm{tot}^2}\right)^2}{2\left(\frac{\sigma_1 \sigma_2}{\sigma_\textrm{tot}}\right)^2}\right] d x \\
& =\frac{1}{\sqrt{2 \pi} \sigma_\textrm{tot}} 
\exp\left[-\frac{z^2}{2 \sigma_\textrm{tot}^2}\right] \notag \\
&~~~\int_{-\infty}^{\infty} \frac{1}{\sqrt{2 \pi} \frac{\sigma_1 \sigma_2}{\sigma_\textrm{tot}}} 
\exp\left[-\frac{\left(x-\frac{\sigma_1^2 z}{\sigma_\textrm{tot}^2}\right)^2}{2\left(\frac{\sigma_1 \sigma_2}{\sigma_\textrm{tot}}\right)^2}\right] d x.
\end{align}
The expression in the integral is a normal density distribution on $x$, and so the integral evaluates to 1. The desired result follows:
\begin{align}
f_\textrm{tot}(z)=\frac{1}{\sqrt{2 \pi} \sigma_\textrm{tot}} \exp \left[-\frac{z^2}{2 \sigma_\textrm{tot}^2}\right].
\end{align}
Therefore, we have the convolution of two consecutive Gaussian process $f_1\sim \mathcal{N}(0,\sigma^2_1)$ and $f_2\sim \mathcal{N}(0,\sigma^2_2)$ follow Gaussian distribution and satisfies $\sigma^2_\textrm{tot} = \sigma^2_1 + \sigma^2_2$.

\section{Proof on Theorem~\ref{theorem:sequential-semcom-distortion}}
\label{appendix:sequential-semcom-distortion}
We have a sequential process as defined in Figure~\ref{fig:Semcom-Process} as a Markov process. Thus, the posterior distribution at the receiver can be represented as follows: 
\begin{align}
    \mathbb{P}(x_i \vert \widetilde{x}_i) = 
    \mathbb{P}(x^s_i \vert x^r_i) 
    \mathbb{P}(x^r_i \vert x^o_i)
    \mathbb{P}(x^o_i \vert \widetilde{x}_i).
\label{eq:SemCom-Markov}
\end{align}
Here, we denote $x^s_i, x^r_i, x^o_i, \widetilde{x}_i$ as the decompressed data (i.e., effected information loss by channel model and compression), data at the receiver (i.e., effected by ), original data, and the empirical ideal data of the global dataset, respectively. 

\begin{figure}[t]
\centerline{\includegraphics[width=1\linewidth]{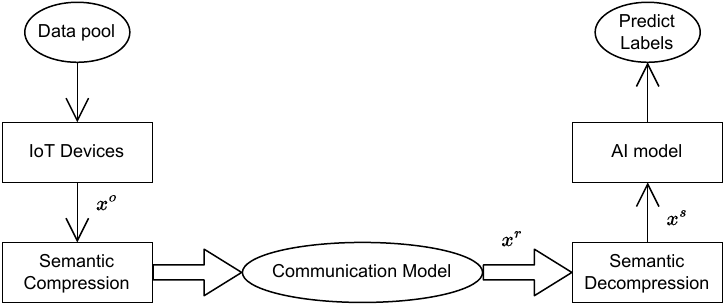}}
\caption{Process of Semantic Communication.}
\label{fig:Semcom-Process}
\end{figure}

Furthermore, the reconstruction loss of a compression system follows the Gaussian distribution $\mathcal{N}(0, \sigma^{\text{sem}})$ \cite{2023-FL-HCFL}. Applying the Gaussian distribution function with Lemma~\ref{lemma:sequential-model-distortion}, the total distortion for the sequential process of the SemCom system is given by: 
\begin{align}
    \sigma^2_\text{tot} =\sigma^2_\text{sem} + \sigma^2_\text{model} +\sigma^2_\text{data}.
\end{align}

\section{Proof on Theorem~\ref{theorem:predict-model-variance}}
\label{appendix:predict-model-variance}
We consider the parameterized hypothesis class $\mathcal{H} = \{ h_\theta \vert \theta \in \mathbf{R}^d \}$, where each member $h_\theta$ is a mapping from $\mathcal{X}_i$ to $\mathcal{Y}_i$ parameterized by $\theta$. With $z_i \triangleq (x_i, y_i) \in \mathcal{Z}_i$, we use $l(h(x_i), y_i)$, shorthand for $l(h_\theta(x_i), y_i)$, to represent the empirical distance between the ground truth $y_i$ and the predicted label $h_\theta(x_i)$. Here, we consider empirical distance $l(h_\theta(x_i), y_i)$ as Euclid distance which is represented as: 
\begin{align}
    l(h_\theta(x_i), y_i) = \Vert h_\theta(x_i) - y_i \Vert
\end{align}
We also assume that the original data $x_i$ is the ideal data (i.e., center on its category of the global dataset) and the hypothesis is well-trained (i.e., the mapping function from $x_i \xrightarrow{h_\theta} y_i$ is correct with $100\%$ confidence).

Taking the hypothesis of distorted data $\hat{x}$ into consideration, we have the Taylor expansion of $h(\hat{x})$ up to the second-order equals to:
\begin{align}
    G(x) &= \sum^{N}_{i=1}\ell(h(\hat{x}_i), y_i) - \sum^{N}_{i=1} \ell(h(\widetilde{x}_i), y_i) \notag \\
    &= 
    \sum^{N}_{i=1}\ell(\hat{x}_i) - \sum^{N}_{i=1} \ell(\widetilde{x}_i) \notag \\
    &=
    \sum^{N}_{i=1}\ell(x_i + \sigma_i) - \sum^{N}_{i=1} \ell(\widetilde{x}_i) \notag \\
    &\overset{(a)}{\leq} 
    \sum^{N}_{i=1}\nabla \ell(x_i) \sigma_i + \sum^{N}_{i=1}\frac{\nabla^2 l(x_i)}{2} \sigma_i^2 \notag \\
    &\leq \frac{L^2}{2} \sum^{N}_{i=1} \sigma_i^2 \notag \leq \frac{L^2}{2N} \widetilde{\sigma}^2.
\end{align}
Here, we denote $\widetilde{\sigma}$ as the variance of the global dataset $\mathcal{D}$ sampled from IoE devices. For simplicity, we define $\ell(h(x_i),y_i)=\ell(x_i)$. The inequality $(a)$ holds due to the assumption~\ref{assumption:dataset-center} of ideal data on well-train model $h^*_\theta$ (i.e., for all $x_i, \exists \rho$ that satisfies  $\nabla h^*_\theta (\widetilde{x_i}) \leq \rho$.
\section{Proof on Theorem - Gradient Dissimilarity}
\label{app:gradient-dissimilarity}
We consider the gradient $\nabla F(\cdot)$ of data point $x_i$ as: $\nabla F(x_i)$. When affected by the distortion, the data is distorted into $x+\sigma$, where $\sigma$ is the distortion of the data via the channel. We have the following lagrange: 
\begin{align}
    \nabla^2 F(x) = \nabla\frac{F(x+\sigma)-\nabla f(x)}{(x+\sigma)-x} \leq L\mathbf{I}.
\end{align}
Therefore, we have: 
\begin{align}
    \nabla F(x+\sigma)-\nabla F(x) \leq L\sigma.
\end{align}
Thus, we have the following gradient dissimilarity: 
\begin{align}
    \nabla F(x+\sigma) \leq \nabla F(x)+L\sigma.
\end{align} 
Therefore, we have: 
\begin{align}
\label{eq:distorted-gradient}
    \Vert\nabla F(x+\sigma)\Vert^2 \leq \Vert\nabla F(x)+L\sigma\Vert^2.
\end{align}

\section{Proof on Theorem - Gradient Inductive}
\label{app:gradient-inductive}
We consider the gradient of distorted data $\hat{x} = x + \sigma$ as $F(\cdot;\hat{x)}$. For ease of the proof, we denote $F(\cdot;\hat{x)}=F(w^{n})$, where $w^n$ is the model parameters $w$ at iteration $n$. We have the followings: 
\begin{align}
\label{eq:taylor1}
    &F(w^{n+1};\hat{x}) = F(w^{n+1}) = F[w^n - \eta \nabla F(w^n)] \notag \\
    &=F(w^n)-\eta\Vert\nabla F(w^n)\Vert^2+\frac{\eta^2}{2}\Vert\nabla F(w^n)\Vert^2\cdot\nabla^2 F(w^n) \notag \\
    &\leq F(w^n)-\eta\Vert\nabla F(w^n)\Vert^2+\eta^2 \frac{L}{2}\Vert\nabla F(w^n)\Vert^2 \notag \\
    &= F(w^n) + \left(\eta^2\frac{L}{2}-\eta\right)\Vert\nabla F(w^n)\Vert^2.
\end{align}
According to L-smooth, we have
\begin{align}
    F(w^{n+1};\hat{x}) \leq &F(w^{n};\hat{x}) + \left(\eta^2\frac{L}{2}-\eta\right)\Vert\nabla F(w^n;\hat{x})\Vert^2 \\
    \leq &F(w^{n};\hat{x}) + \left(\eta^2\frac{L}{2}-\eta\right) \Vert\nabla F(w^n;x)+L\sigma\Vert^2 \notag ,
\end{align}
and according to $\mu$-convex, we have
\begin{align}
    F(w^{n+1};\hat{x}) \geq &F(w^{n};\hat{x}) + \left(\eta^2\frac{\mu}{2}-\eta\right)\Vert\nabla F(w^n;\hat{x})\Vert^2 \\
    \geq &F(w^{n};\hat{x}) + \left(\eta^2\frac{\mu}{2}-\eta\right) \Vert\nabla F(w^n;x)+L\sigma\Vert^2 \notag .
\end{align}
Taking the Generalization Gap into consideration, we have the followings:
\begin{align}
    F&(w^{n+1};\hat{x}) - F(w^{n+1};x) \notag \\
    \leq &\left[F(w^{n};\hat{x}) + \left(\eta^2\frac{L}{2}-\eta\right)\Vert\nabla F(w^n;\hat{x})\Vert^2 \right] \notag \\
    - &\left[F(w^{n};x) + \left(\eta^2\frac{\mu}{2}-\eta\right)\Vert\nabla F(w^n;x)\Vert^2\right] 
    \notag \\
    \leq &\left[F(w^{n};\hat{x}) - F(w^{n};x)\right] \notag \\
    + &\left[\left(\eta^2\frac{L}{2}-\eta\right)\Vert\nabla F(w^n;\hat{x})\Vert^2 - \left(\eta^2\frac{\mu}{2}-\eta\right)\Vert\nabla F(w^n;x)\Vert^2\right] 
    \notag \\ 
    \leq &\left[F(w^{n};\hat{x}) - F(w^{n};x)\right] \notag \\
    + &\left[\left(\eta^2\frac{L}{2}-\eta\right)\Vert\nabla F(w^n;x)+L\sigma\Vert^2\right. \notag \\
    &- \left.\left(\eta^2\frac{\mu}{2}-\eta\right)\Vert\nabla F(w^n;x)\Vert^2\right] 
    \notag \\
    \leq &\left[F(w^{n};\hat{x}) - F(w^{n};x)\right] \notag \\
    + &\left[\left(\eta^2\frac{L}{2}-\eta\right)\left[\Vert\nabla F(w^n;x)\Vert^2+\Vert L\sigma\Vert^2\right] \right. \notag \\
    &- \left.\left(\eta^2\frac{\mu}{2}-\eta\right)\Vert\nabla F(w^n;x)\Vert^2\right] \notag \\ 
    \leq &\left[F(w^{n};\hat{x}) - F(w^{n};x)\right] \\
    + &\left[\left(\eta^2\frac{L}{2}-\eta\right)\Vert L\sigma\Vert^2 + \frac{\eta^2}{2}\left(L-\mu\right)\Vert\nabla F(w^n;x)\Vert^2\right]. \notag 
\end{align}
By choosing the $L$-smooth and $\mu$-strongly convex such that $L=\mu$, we have:
\begin{align}
    &R(w^{n}) = F(w^{n};\hat{x}) - F(w^{n};x) \notag \\
    &\leq F(w^{n-1};\hat{x}) - F(w^{n-1};x) + \left( \eta^2\frac{L}{2}-\eta \right)(L\sigma)^2 \notag \\
    &= F(w^0;\hat{x})-F(w^0;x) + n\left( \eta^2\frac{L}{2}-\eta \right)(L\sigma)^2.
\end{align}
Here, we have $F(w^0;\hat{x}) = F(w^0;x)$ (because at the initial weight $w^0$, the loss on two data is equivalent). Thus, we have: 
\begin{align}
    F(w^n;\hat{x})-F(w^n;x)\leq n\left( \eta^2\frac{L}{2}-\eta \right)(L\sigma)^2.
\end{align}

\begin{figure}[t]
\centerline{\includegraphics[width=1\linewidth]{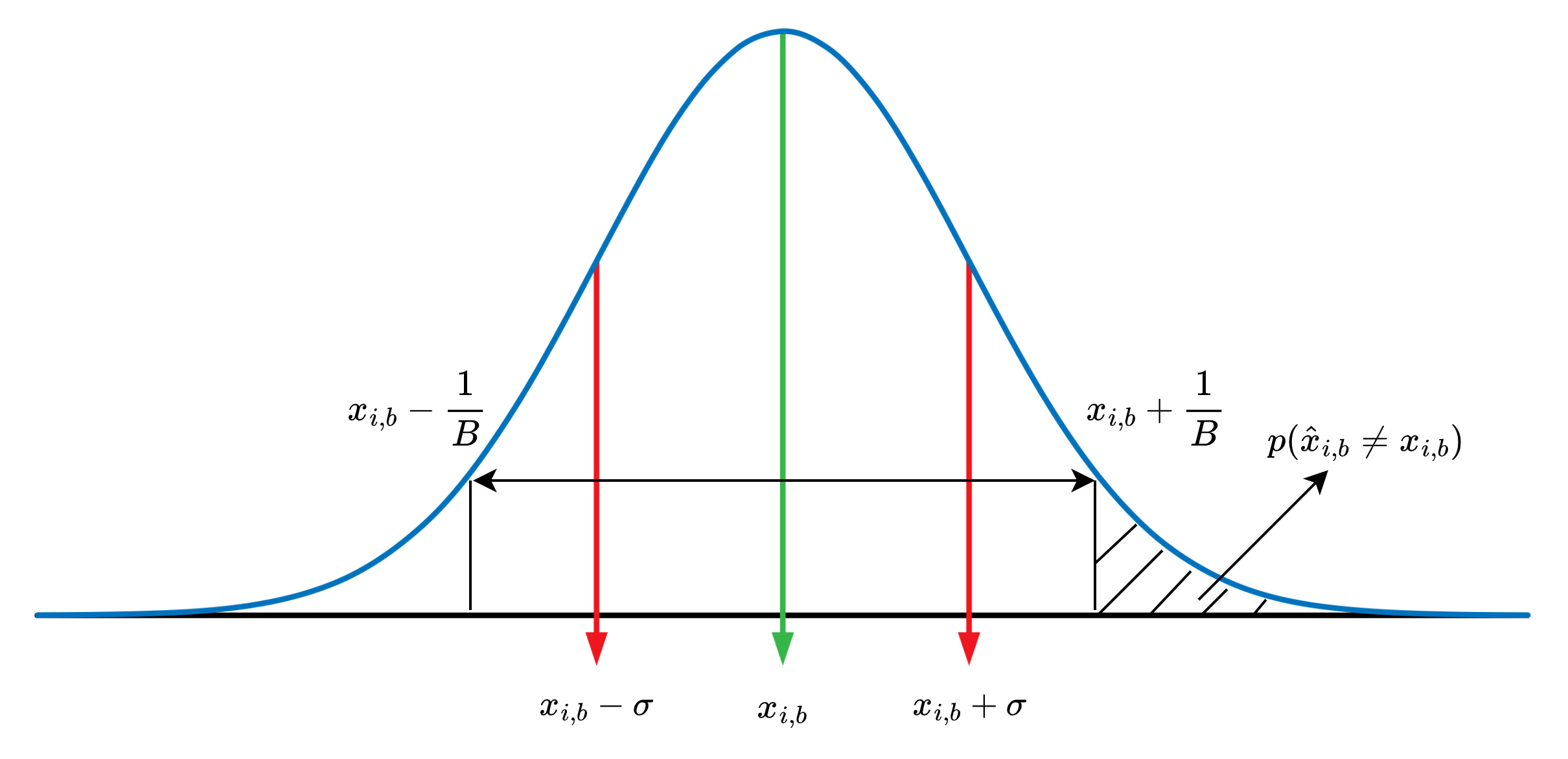}}
\caption{Illustration on probability decision boundary induced by Gaussian noise.}
\label{fig:AI-DecisionBoundary}
\end{figure}

\section{Proof on Lemma - Total Variation}
\label{app:total-variation}
\begin{align}
    \textrm{TV}(x_u,\hat{x}) 
    &= \underset{v\in \mathcal{V}}{\textrm{sup}}\left[ P(x) - P(\hat{x}) \right] \notag \\
    &= \underset{v\in \mathcal{V}}{\textrm{sup}}\left[ \underset{x\in\mathcal{X}}{\int}p(x)dx - \underset{x\in\mathcal{X}}{\int}q(x)dx) \right] \notag \\
    &= \underset{v\in \mathcal{V}}{\textrm{sup}}\left[ \underset{x\in\mathcal{X}}{\int}\langle p(x) - q(x) \rangle dx \right].
\end{align}
Here, we have the transportation policy $v$ is sampled from the set of possible transportation policy $\mathcal{V}$. As mentioned in \cite{2007-DL-GreedyLayerWise}, the input $p(x_i)$ is modeled by the set of $p(x_{i,b})$ where $b$ is the bit index in the data $i$. Therefore, we can have the following: 
\begin{align}
    \textrm{TV}(x,\hat{x}) 
    &= \underset{v\in \mathcal{V}}{\textrm{sup}}\left\vert P(x) - P(\hat{x}) \right\vert \notag \\
    &= \underset{v\in \mathcal{V}}{\textrm{sup}}\left[\underset{x\in x}{\int}\vert p(x) - q(x) \vert dx \right] \notag \\ 
    &= \overset{\vert x\vert}{\underset{i=0}{\sum}} \overset{B}{\underset{b=0}{\sum}} \vert p(x_{i,b})-p(\hat{x}_{i,b})\vert,
\end{align}
where $p(x_{i,b}), p(\hat{x}_{i,b})$ represent the element-wise probability of the data $x_i$ and $\hat{x}_i$ at bit index $b$, respectively. To be more specific, $p(x_{i,b})$ indicates the probability that bit $x_i$ is $1$. Due to the distortion effect, the decision boundary of the data can be moved by noise $v\sim \mathcal{N}(0, \sigma^2)$, where $\sigma^2$ is the variance of the Gaussian noise with mean $0$. We have the variance can be represented as: 
\begin{align}
    \sigma^2 &= \frac{1}{NB}\sum^{N}_{i=1}\sum^{B}_{b=1} (x_{i,b}-\hat{x}_{i,b})^2 = MSE(x_i, \hat{x}_i).
\end{align}
Thus, we can have the data point being induced by the Gaussian noise can be represented as Figure~\ref{fig:AI-DecisionBoundary}. Thus, we have the probability of $p(x_{i,b})-p(\hat{x}_{i,b} = x_{i,b}) = p(\hat{x}_{i,b}\neq p(x_{i,b}))$. Therefore, we can have the followings: 
\begin{align}
    &p(x_{i,b})-p(\hat{x}_{i,b}) 
    = p(\hat{x}_{i,b}\neq x_{i,b}) = p(\frac{v}{\sigma^2} \geq W) \notag \\
    &= \int^{\infty}_{W} \frac{1}{\sqrt{2\pi}\sigma}e^{-\left(\frac{x}{2\sigma}\right)^2} 
    \leq \frac{1}{\sqrt{2\pi}\sigma}e^{-\left(\frac{W}{2\sigma}\right)^2}. 
\end{align}

\section{Proof on Theorem - Inference Divergence}
\label{appendix:inference-divergence}
We consider the gradient of distorted data $\hat{x} = x + \sigma$ as $F(\cdot;\hat{x)}$. For ease of the proof, we denote $F(\cdot;\hat{x)}=F(w^{n})$, where $w^n$ is the model parameters $w$ at iteration $n$. We have the following: 
\begin{align}
    p(\mathbf{\hat{t}}) 
    &= p(\mathbf{t}\vert \mathbf{\hat{s}})p(\mathbf{\hat{s}}) = p(\mathbf{t}\vert s_u)p(\mathbf{\hat{s}}).
\end{align}
We consider the long-term performance of the SemCom system. Thus, we consider the expected output as follows: 
\begin{align}
    \mathbb{E}\left[p(\mathbf{\hat{t}})\right]
    &= \mathbb{E}\left[p(\mathbf{t}\vert s_u)p(\mathbf{\hat{s}})\right] \notag \\
    &= \mathbb{E}\left[p(\mathbf{t}\vert s_u)p(s_u)\right] + \mathbb{E}\left[p(\mathbf{t}\vert s_u)\langle p(\mathbf{\hat{s}}) - p(s_u) \rangle\right] \notag \\
    &= \mathbb{E}\left[p(\mathbf{t})\right] + p(\mathbf{t}\vert s_u)\mathbb{E}\left[ p(\mathbf{\hat{s}}) - p(s_u) \right].
\end{align}
Applying the inequality in Lemma~\ref{lemma:total-variation}, we have the following: 
\begin{align}
    \mathbb{E}\left[p(\mathbf{\hat{t}}) - p(\mathbf{t})\right]
    &\leq \frac{1}{\sqrt{2\pi}\sigma} p(\mathbf{t}\vert s_u) \exp{-\left(\frac{W}{2\sigma}\right)^2}.
\end{align}

\section{Proof on Theorem - Domain Aggregation}
\label{appendix:global-domain-estimation}

\begin{align}
    \sigma^2_\textrm{tot} &= MSE(x_i, \hat{x}_i) \notag \\
    &= \Vert x_{i,b} - \hat{x}_{i,b} \Vert^2_2 = \mathbb{E}[(x - \mu)^2] = \mathbb{E}[x^2] \notag \\
    &= \mathbb{E}_{u\in U}\left[\mathbb{E}_{x\in X_u}[x^2]\right] \notag \\
    &= \frac{1}{\sum^U_{u=1} D_u}\sum D_u \mathbb{E}_{x\in X_u}[x^2] \notag \\
    &= \frac{1}{D} \sum^U_{u=1} D_u \sigma^2_{\textrm{tot},u}.
\end{align}

\textcolor{duong}{\section{Proof on lemma~\ref{lemma:ECS-P}: Boundaries of Ambiguous Constraints}\label{appendix:ECS_x}
The function of the user's neural network is defined as follows:
\begin{align}
Y = \mathbf{W}^{\mathrm{max}}_u \sigma(\mathbf{W}(x)),
\end{align}
where $\mathbf{W}(x)$ is the neural network. The sigma is bounded by $0<\sigma<1$. Therefore, $0<Y<\mathbf{W}^{\mathrm{max}}$.
For instance, $P_\mathrm{ECS}$ of local computation capacity $f_u$ is defined as follows: 
    \begin{align}
    P_\mathrm{ECS}(f_u) = \abs{\max\left\langle f_u - f^{\mathrm{max}}_u, 0 \right\rangle} + \abs{\max\left\langle -f_u, 0 \right\rangle} \geq 0,
    \end{align}
\section{Proof on Theorem~\ref{theorem:ECS_reward}: Boundaries of Explicit Constraints}\label{appendix:ECS_reward}
Using constraints in the format of Definition~\ref{def:ECS-P}, we obtain the total penalty for the EI-integrated DRL system as follows: 
\begin{align}
    \mathbf{r}_\mathrm{DRL-EI} = R - \mathcal{P}_1 - \mathcal{P}_2 - \sum_i P_{\mathrm{ECS-EI},i},
\end{align}
and the total penalty for DRL system without EI design as follows: 
\begin{align}
    \mathbf{r}_\mathrm{DRL} = R - \mathcal{P}_1 - \mathcal{P}_2 - \sum_i P_{\mathrm{ECS},i},
\end{align}
Therefore, we have the following inequalities for the optimal reward: 
\begin{align}
    \mathbf{r}^*_\mathrm{DRL-EI}
    &= R^* - \mathcal{P}_1^* - \mathcal{P}_2^* - \sum_i P^*_{\mathrm{ECS-EI},i} \notag \\ 
    &\overset{(a)}{=} R^* - \mathcal{P}_1^* - \mathcal{P}_2^* \notag \\ 
    &\geq R^* - \mathcal{P}^*_1 - \mathcal{P}^*_2 - \sum_i P^*_{\mathrm{ECS},i} \notag \\
    &\geq \mathbf{r}^*_\mathrm{DRL},
\end{align}
where the equality (a) holds due to Lemma~\ref{lemma:ECS-P}. In other words, the optimal penalty following the ECS is determined as $\mathbf{r}^*_\mathrm{DRL-EI} \geq \mathbf{r}^*_\mathrm{DRL}$.}
\end{document}